\documentclass{article}

\usepackage{microtype}
\usepackage{graphicx}
\usepackage{subcaption}
\usepackage{booktabs} 

\usepackage{listings}
\usepackage{xcolor}
\usepackage{threeparttable}
\lstdefinestyle{cppstyle}{
  language=C++,
  backgroundcolor=\color{gray!10},
  basicstyle=\ttfamily\footnotesize,
  keywordstyle=\color{blue!90}\bfseries,
  commentstyle=\color{gray!70}\itshape,
  stringstyle=\color{orange!90!black},
  numbers=left,
  numberstyle=\tiny\color{gray},
  stepnumber=1,
  numbersep=6pt,
  frame=single,
  rulecolor=\color{gray!30},
  breaklines=true,
  tabsize=2,
  showstringspaces=false,
  captionpos=b
}
\usepackage{hyperref}

\usepackage[preprint]{icml2026}

\usepackage{amsmath}
\usepackage{amssymb}
\usepackage{mathtools}
\usepackage{amsthm}

\newcommand{\bx}{\boldsymbol{x}}
\newcommand{\barg}{\bar{g}}
\usepackage[capitalize,noabbrev]{cleveref}

\theoremstyle{plain}
\newtheorem{theorem}{Theorem}[section]

\newtheorem{lemma}[theorem]{Lemma}

\theoremstyle{definition}

\newtheorem{assumption}[theorem]{Assumption}
\theoremstyle{remark}
\newtheorem{remark}[theorem]{Remark}
\newtheorem{example}[theorem]{Example}
\usepackage[textsize=tiny]{todonotes}
\icmltitlerunning{Why is Normalization Preferred? A Worst-Case Complexity Theory for SPSGD under Heavy-Tailed Noise}

\begin{document}

\twocolumn[
  \icmltitle{Why is Normalization Preferred? A Worst-Case Complexity Theory for Stochastically Preconditioned SGD under Heavy-Tailed Noise}



  \icmlsetsymbol{equal}{*}

  \begin{icmlauthorlist}
    \icmlauthor{Yuchen Fang}{UCB_Math}
    \icmlauthor{James Demmel}{UCB_EECS}
    \icmlauthor{Javad Lavaei}{UCB_IEOR}
  \end{icmlauthorlist}

  \icmlaffiliation{UCB_Math}{Department of Mathematics, University of California, Berkeley, CA, USA}
 \icmlaffiliation{UCB_EECS}{Department of EECS, University of California, Berkeley, CA, USA}
  \icmlaffiliation{UCB_IEOR}{Department of IEOR, University of California, Berkeley, CA, USA}

  \icmlcorrespondingauthor{Yuchen Fang}{yc\_fang@berkeley.edu}

  \icmlkeywords{Machine Learning, ICML}

  \vskip 0.3in
]



\printAffiliationsAndNotice{}  

\begin{abstract}
We develop a worst-case complexity theory for \emph{stochastically preconditioned stochastic gradient descent} (SPSGD) and its accelerated variants under heavy-tailed noise, a setting that encompasses widely used adaptive methods such as Adam, RMSProp, and Shampoo. We assume the stochastic gradient noise has a finite $p$-th moment for some $p \in (1,2]$, and measure convergence after $T$ iterations. While clipping and normalization are parallel tools for stabilizing training of SGD under heavy-tailed noise, there is a fundamental separation in their worst-case properties in stochastically preconditioned settings. We demonstrate that normalization guarantees convergence to a first-order stationary point at rate $\mathcal{O}(T^{-\frac{p-1}{3p-2}})$ when problem parameters are known, and $\mathcal{O}(T^{-\frac{p-1}{2p}})$ when problem parameters are unknown, matching the optimal rates for normalized SGD, respectively. In contrast, we prove that clipping may fail to converge in the worst case due to the statistical dependence between the stochastic preconditioner and the gradient estimates. To enable the analysis, we develop a novel vector-valued Burkholder-type inequality that may be of independent interest. These results provide a theoretical explanation for the empirical preference for normalization over clipping in large-scale model training.
\end{abstract}

\section{Introduction}\label{intro}
We consider stochastic optimization problems of the form
\begin{equation}\label{Problem1}
    \min_{\bx \in \mathbb{R}^d} f(\bx) := \mathbb{E}_{\xi \sim \mathcal{P}}[F(\bx,\xi)],
\end{equation}
where $f: \mathbb{R}^d \to \mathbb{R}$ is the nonconvex objective function and $F(\bx,\xi)$ denotes a stochastic realization with sample $\xi$ drawn from a distribution $\mathcal{P}$. Stochastic gradient descent (SGD) \citep{Robbins1951Stochastic} and its variants remain the workhorse for large-scale machine learning due to their simplicity, scalability, and strong empirical performance. Under classical assumptions of unbiased gradients with bounded variance, SGD admits well-understood convergence guarantees, including the optimal rate $\mathcal{O}(T^{-\frac{1}{4}})$ for finding first-order stationary points in smooth nonconvex problems.

Building upon SGD, numerous variants have been proposed to accelerate convergence and improve robustness. A classical and widely used extension is momentum SGD (MSGD) \citep{Polyak1964Some, Liu2020Improved}, which updates the iterate according to
\begin{equation}\label{rule2}
\bx_{k+1} = \bx_k - \eta m_k,
\qquad
m_k = \theta m_{k-1} + (1-\theta)\barg_k,
\end{equation}
where $\eta$ is the learning rate, $\barg_k$ is a stochastic gradient at $\bx_k$, and $\theta \in (0,1)$ is the momentum parameter. Both SGD and MSGD enjoy strong theoretical guarantees and empirical success when the gradient noise has bounded variance.

However, the bounded-variance assumption is often violated in modern machine learning. Empirical studies in image classification \citep{Battash2024Revisiting}, large language model training \citep{Zhang2020Adaptive,Ahn2024Linear}, and reinforcement learning \citep{Garg2021Proximal} provide compelling evidence that the noise in gradient estimation is often heavy-tailed, where only a finite $p$-th moment exists for some $p\in(1,2]$. In such regimes, vanilla SGD and momentum SGD may exhibit instability or even divergence, and standard convergence guarantees might no longer apply \cite{Gorbunov2020Stochastic}. This has motivated a growing body of work on stochastic optimization under heavy-tailed noise.

Two simple yet widely used techniques for stabilizing SGD under heavy-tailed noise are clipping (ClipSGD) \citep{Zhang2020Adaptive} and normalization (NSGD) \citep{Cutkosky2020Momentum}. Clipping truncates gradient estimates exceeding a prescribed threshold, whereas  normalization rescales gradient estimates to have unit norm. It has been shown that both approaches are known to achieve the optimal convergence rate $\mathcal{O}(T^{-\frac{p-1}{3p-2}})$ in nonconvex optimization. Despite their theoretical parity in the vanilla SGD setting, a significant disconnect remains between theory and practice. Clipping often suffers from practical limitations: (1) determining the optimal clipping threshold requires precise knowledge of algorithmic parameters that are difficult to estimate; and (2) while theory often prescribes large, increasing thresholds, practitioners typically employ small, fixed thresholds \citep{Zhang2022Opt,Liu2024Sophia}. In contrast, normalization is easier to tune and exhibits greater robustness in practice. Therefore, the idea of normalization has been applied in large-scale machine learning extensively \citep{you2017large,you2018imagenet,you2019large,touvron2023llama}. 

In this work, we show that this gap between clipping and normalization becomes fundamental once we move beyond vanilla SGD to Stochastically Preconditioned SGD (SPSGD) and its accelerated variants, which has the following update rule
\begin{equation}\label{rule3}
\bx_{k+1} = \bx_k - \eta D_k m_k,
\qquad
m_k = \theta m_{k-1} + (1-\theta)\barg_k,
\end{equation}
where $D_k$ is a symmetric positive definite stochastic preconditioning matrix that may depend on past gradients or curvature information.




The update rule in \eqref{rule3} subsumes a wide range of widely used optimization methods, including AdaGrad \citep{Duchi2011Adaptive}, Adam \citep{Kingma2014Adam}, RMSProp \citep{Tieleman2012Lecture}, K-FAC \citep{martens2015optimizing}, Shampoo \citep{Gupta2018Shampoo}, and stochastic Newton, quasi-Newton, and trust-region methods \cite{bottou2018optimization,Wang2017Stochastic,Curtis2022Fully}. We discuss in detail how \eqref{rule3} recovers those methods in Section \ref{sec2:preconditioned_methods}.

Despite their empirical success, the theoretical understanding of SPSGD remains limited. Existing analyses often either ignore the randomness of $D_k$ or impose strong independence assumptions between $D_k$ and the stochastic gradient, which are violated in most adaptive and second-order methods. Moreover, nearly all existing results focus on light-tailed noise with bounded variance, leaving the heavy-tailed regime largely unexplored.

In this paper, we investigate how step clipping and step normalization\footnote{Throughout the paper, we use the terms step clipping and step normalization since what we normalize and clip can be the stochastic preconditioned gradient, instead of gradient alone.} behave when combined with stochastic preconditioning under heavy-tailed noise. Our main finding is a sharp worst-case separation:
\begin{itemize}
    \item \textit{Normalization remains robust}: when step normalization is applied to SPSGD, the method converges to a first-order stationary point under heavy-tailed noise, achieving \textit{at least} the same optimal rates as normalized SGD, both when problem parameters are known and when they are unknown.
    \item \textit{Clipping may fail}: in contrast, when step clipping is applied to SPSGD, the method may fail to converge in the worst case. This failure arises from the statistical dependence between the stochastic preconditioner and the gradient estimate, which introduces a persistent covariance that clipping cannot eliminate.
\end{itemize}

This separation \textit{does not} appear in vanilla SGD and provides a novel theoretical explanation for the wide applications of normalization in modern large-scale machine learning.

Our contributions are summarized as follows:
\begin{enumerate}

\item We establish the worst-case iteration complexity guarantees for SPSGD and its accelerated variants with step normalization under heavy-tailed noise. Combined with momentum, normalized SPSGD achieves the rate $\mathcal{O}(T^{-\frac{p-1}{3p-2}})$ when problem parameters are known, and $\mathcal{O}(T^{-\frac{p-1}{2p}})$ when parameters are unknown. Both rates match the optimal rates for NSGD under heavy-tailed noise when algorithmic parameters are known and unknown, respectively. Importantly, these rates are derived using minimal structural assumptions on $D_k$ and thus serve as unified upper bounds. 

\item We demonstrate that, in the worst case, SPSGD combined with step clipping may fail to converge. This failure is intrinsic and stems from the statistical dependence between the preconditioner and the gradient noise. We also provide an explanation from a geometric perspective to emphasize the fundamental difference between step normalization and step clipping.

\item To enable our analysis, we establish a new vector-valued Burkholder-type inequality, generalizing the scalar result of \citep{Fang2025High} and yielding sharper constants than existing bounds \citep{Hubler2024Gradient}. This inequality may be of independent interest for analyzing stochastic optimization algorithms.
\end{enumerate}

\subsection{Related work}

\noindent\textbf{Gradient clipping} is widely used to stabilize the
training in various fields of machine learning \citep{Pascanu2013Difficulty,Schulman2017Proximal,Zhang2020Adaptive}. Recently a number of works provide convergence guarantees and extend the algorithm design in various settings, including heavy-tailed cases \citep{Nazin2019Algorithms,Gorbunov2020Stochastic,Gorbunov2024Methods,Puchkin2024Breaking,Sadiev2023High,nguyen2023improved}. Most of the above mentioned works use increasing and iteration dependent clipping parameters, which contrast with the clipping techniques that we use in practice. Recently, \citet{koloskova2023revisiting} offer
a new analysis with a constant clipping
threshold under bounded-variance. However, their proof seems challenging to extend to heavy-tailed setting.

\noindent \textbf{Gradient normalization} was first proposed by \citep{nesterov1984minimization} and then extensively generalized, e.g., see \cite{levy2017online,hazan2015beyond,Yang2023Two,hubler2024parameter,levy2016power}. Its extension to nonconvex optimization and the removal of large-batch requirements were achieved by \cite{Cutkosky2020Momentum} through the introduction of momentum.
However, all of the above works rely on strong assumptions imposed on noise, most notably bounded variance. Under heavy-tailed noise, \citet{Cutkosky2021High} studied combinations of gradient normalization and gradient clipping with iteration-dependent clipping thresholds. More recently, \cite{Sun2025Revisiting} showed that gradient normalization combined with momentum is sufficient to achieve optimal convergence rates in expectation. Their analysis, however, relies on an individual Lipschitz condition, which implicitly assumes bounded gradient noise and thus restricts the setting to light-tailed distributions. Concurrently, \cite{Hubler2024Gradient} established convergence rates and sample complexities under heavy-tailed noise and proved their tightness, while \cite{Liu2025Nonconvex} considered generalized $p$-BCM and $(L_0,L_1)$-smooth conditions and obtained the same optimal rates. Nevertheless, all of the aforementioned analyses are restricted to vanilla SGD and do not account for the additional challenges introduced by stochastic preconditioning.

\textbf{Notation.}  We use $\|\cdot\|$ to denote $\ell_2$ norm for vectors and operator norm for matrices. We use $\odot$ to denote entrywise multiplication of two vectors, and $\otimes$ to denote the Kronecker product. Given the iterate $\bx_k$, we denote $\mathbb{E}_k[\cdot] \coloneqq\mathbb{E}[\cdot \mid \bx_{k}]$.

\section{Stochastically Preconditioned Methods}\label{sec2:preconditioned_methods}

In this section, we demonstrate that the unified update rule \eqref{rule3} encompasses a broad class of widely used stochastic optimization algorithms, all of which fall within the scope of our theoretical framework. While we do not aim to enumerate all possible methods, this section and Table~\ref{tab:preconditioned_methods} in Appendix \ref{appendix:table} summarize representative examples.

\textbf{SGD and momentum SGD.}
As a baseline, when $D_k = I$, the update rule \eqref{rule3} recovers stochastic gradient descent with momentum. When $\theta=0$, it further reduces to vanilla SGD.

\textbf{Adaptive learning-rate methods} employ coordinate-wise learning rates that depend on historical gradient information, which can be equivalently expressed through diagonal preconditioning matrices.

\textbf{AdaGrad} \citep{Duchi2011Adaptive} accumulates squared gradients
$G_k = \sum_{i=0}^k \barg_i \odot \barg_i$ and updates $\bx_{k+1} = \bx_k - \frac{\eta}{\sqrt{G_k + \epsilon}} \odot \barg_k$.
This corresponds to setting $D_k = \operatorname{diag}\!\big(1/\sqrt{G_k+\epsilon}\big)$ and $\theta=0$.

\textbf{RMSProp} \citep{Tieleman2012Lecture} replaces cumulative sums with exponential moving averages
$v_k = \beta v_{k-1} + (1-\beta)\barg_k^2$, yielding the update $\bx_{k+1} = \bx_k - \frac{\eta}{\sqrt{v_k + \epsilon}} \odot \barg_k$,
which corresponds to $D_k = \operatorname{diag}\!\big(1/\sqrt{v_k+\epsilon}\big)$ and $\theta=0$.

\textbf{Adam} \citep{Kingma2014Adam} combines momentum and adaptive scaling by maintaining $m_k = \beta_1 m_{k-1} + (1-\beta_1)\barg_k$, $v_k = \beta_2 v_{k-1} + (1-\beta_2)\barg_k^2$,
with the update $\bx_{k+1} = \bx_k - \frac{\eta}{\sqrt{v_k + \epsilon}} \odot m_k$.
This leads to $D_k = \operatorname{diag}\!\big(1/\sqrt{v_k+\epsilon}\big)$ and $\theta=\beta_1$.

\textbf{Adafactor} \citep{zhai2022scaling,zhao2024deconstructing} reduces memory usage by replacing
$v_k$ with a low-rank approximation $v_k'$, resulting in the diagonal preconditioner
$D_k = \operatorname{diag}\!\big(1/\sqrt{v_k' + \epsilon}\big)$.

\textbf{Kronecker-structured preconditioning methods} extend diagonal scaling by employing structured matrix preconditioners based on Kronecker factorizations, allowing richer curvature information to be exploited efficiently.

\textbf{K-FAC} \citep{martens2015optimizing} approximates the Fisher information matrix by
$F_k \approx A_k \otimes B_k$, where $A_k,B_k$ are symmetric positive definite matrices. This approximation yields the update $\bx_{k+1} = \bx_k - \eta (A_k \otimes B_k)^{-1} m_k$, which corresponds to choosing $D_k = (A_k \otimes B_k)^{-1}$.

\textbf{Shampoo} \citep{Gupta2018Shampoo} maintains Kronecker-factored second-moment estimates $L_k = \beta L_{k-1} + (1-\beta) G_k G_k^T$, $R_k = \beta R_{k-1} + (1-\beta) G_k^T G_k$
and updates $X_{k+1} = X_k - \eta L_k^{-1/4} G_k R_k^{-1/4}$.
Using properties of the Kronecker product, this can be equivalently written as $\operatorname{vec}(X_{k+1})
=
\operatorname{vec}(X_k)
-
\eta (L_k \otimes R_k)^{-1/4} \operatorname{vec}(G_k)$,
which corresponds to $D_k = (L_k \otimes R_k)^{-1/4}$ and $\theta=0$.


\textbf{Stochastic second-order methods} incorporate full or approximate curvature information by choosing $D_k$
as an approximation to the inverse Hessian.

\textbf{Stochastic Newton methods} \citep{bottou2018optimization} update $\bx_{k+1} = \bx_k - \eta (H_k + \lambda I)^{-1} m_k$,
where $H_k$ is a stochastic Hessian estimate and $\lambda \ge 0$ ensures positive definiteness.
Thus $D_k = (H_k + \lambda I)^{-1}$.

\textbf{Quasi-Newton methods} approximate the Hessian using low-rank updates, such as stochastic BFGS \citep{Wang2017Stochastic}. In this case, $D_k$ is given by the inverse of the Hessian approximation, with the precise form depending on the Quasi-Newton formula.

\textbf{Stochastic trust-region methods} \citep{Curtis2022Fully} compute steps by approximately solving
a local quadratic model within a trust region. When solved exactly, the resulting step takes the form
$-\eta (H_k + \lambda I)^{-1} m_k$, where $\lambda$ depends on the trust-region radius. When solved inexactly, for example using Krylov subspace methods such as Steihaug method \citep[Algorithm 7.2]{Nocedal2006Numerical}, the step can be expressed as a linear combination of
$m_k, H_k^{-1}m_k, H_k^{-2}m_k,\ldots$, which again fits \eqref{rule3} with a data-dependent and potentially stochastic preconditioner $D_k$.

\section{Takeaway: A Geometric Perspective}

Before presenting our technical analysis, we offer an intuitive geometric explanation for why \emph{step normalization} remains robust under stochastic preconditioning, whereas \emph{step clipping} faces fundamental stability issues.

At each iteration, the stochastic gradient $\barg_k$ is transformed by a symmetric positive definite (SPD) preconditioner $D_k$. This yields a preconditioned direction $D_k \barg_k$, which can be highly anisotropic and variable in magnitude due to curvature and data-dependent scaling.

\textbf{Step normalization.}
Under step normalization, the preconditioned update is strictly projected onto the surface of a sphere. By enforcing a fixed norm, normalization effectively decouples the update magnitude from the scale of the preconditioner. Geometrically, this constrains the optimization trajectory to a stable manifold, preventing the ``magnitude noise" of $D_k$ from propagating into the update step. 

\textbf{Step clipping.}
In contrast, step clipping allows the update to lie either on the boundary or \emph{anywhere within} the interior of a ball of radius $\tau$. Crucially, to achieve optimal convergence rate, theoretical analyses require this threshold $\tau$ to grow with the time horizon $T$. Geometrically, this means the update region is an expanding ball. Because clipping retains magnitude information when the gradient is small, the statistical dependence between $D_k$ and $\barg_k$ introduces a bias within this expanding interior.

\section{Normalization Ensures Convergence}\label{sec3}

In this section, we demonstrate that step normalization ensures the convergence of Stochastically Preconditioned SGD (SPSGD) to a first-order stationary point in expectation. Furthermore, we establish that in the worst case—where minimal structural knowledge of the stochastic preconditioner is assumed—the convergence rate matches the optimal rate of Normalized SGD (NSGD). 

\begin{algorithm}[t]
\caption{SPSGD with Step Normalization}
\label{alg:nsgdm}
\begin{algorithmic}[1]
\STATE \textbf{Input:} Initial point $\bx_1 \in \mathbb{R}^d$, momentum parameter $\theta \in (0,1)$, learning rate $\eta > 0$.
\FOR{$k = 1$ to $T$}
    \STATE Obtain an unbiased gradient estimate $\barg_k$.
    \STATE Update the momentum (with $m_0=\barg_1$):
    \begin{equation*}
        m_k = \theta m_{k-1} + (1 - \theta) \barg_k.
    \end{equation*}
    \STATE Construct a (potentially stochastic) preconditioning matrix $D_k$ and compute $D_km_k$.
    \STATE Perform normalization and update the iterate:
    \begin{equation*}
        x_{k+1} = x_k - \eta \frac{D_km_k}{\|D_km_k\|}.
    \end{equation*}
    \ENDFOR
\end{algorithmic}
\end{algorithm}

We summarize 
Algorithm~\ref{alg:nsgdm}. The algorithm follows the standard design of NSGD \citep{Liu2025Nonconvex,Sun2025Revisiting}, with the key distinction being the incorporation of $D_k$. We include momentum following the algorithmic framework of \citet{Cutkosky2020Momentum}, where momentum is shown to improve the performance of NSGD without requiring large per-iteration batch sizes. In our analysis, a single sample per iteration suffices.

For analysis, we impose the following assumptions.
\begin{assumption}\label{assumption1}
There exists
$f_{*} = \inf_{\bx \in \mathbb{R}^d} f(\bx) > -\infty$.
\end{assumption}

\begin{assumption}[$L$-smoothness]\label{assumption2}
There exists $L > 0$ such that for any $\bx, \boldsymbol{y} \in \mathbb{R}^d$, we have
\begin{equation*}
  \|\nabla f(\boldsymbol{x}) - \nabla f(\boldsymbol{y})\| \leq L \|\bx - \boldsymbol{y}\|.  
\end{equation*}
\end{assumption}

\begin{assumption}[$p$-BCM condition]\label{assumption3}
At the $k$-th iteration, the gradient estimate $\barg_k$ is unbiased with bounded $p$-th moment, that is,
$\mathbb{E}_k[\barg_k ] = \nabla f(\bx_k)$ and
\begin{equation*}
\mathbb{E}_k \big[ \|\barg_k - \nabla f(\bx_k)\|^p \big] \leq \sigma^p
\end{equation*}
for some $p \in (1,2]$ and $\sigma \geq 0$.
\end{assumption}

The $L$-smoothness assumption is standard in the analysis of stochastic optimization \citep{Zhang2020Adaptive,Sun2025Revisiting,Hubler2024Gradient}. Recently, \citet{Liu2025Nonconvex} imposed a weaker $(L_0,L_1)$-smoothness assumption. However, this assumption complicates the analysis when algorithmic parameters are unknown, as the learning rate $\eta$ still depends on $L_1$, and current literature has not eliminated this dependence. \citet{Liu2025Nonconvex} also considered a generalized $p$-BCM condition, but under this condition, a batch of samples is required per iteration to achieve the optimal convergence rate, with the batch size depending on $p$. Thus, the analysis in the parameter-unknown setting still depends on $p$.
  
\begin{assumption}[Bounded condition number of $D_k$]\label{assumption4}
The stochastic preconditioner $D_k$ is symmetric positive definite. There exist (possibly stochastic and iteration-dependent) constants $0<m_{D,k}<M_{D,k}<\infty$ such that $m_{D,k}\cdot I \preceq D_k \preceq M_{D,k}\cdot I$. For each run of the algorithm, there exists a deterministic constant $\kappa_D<\infty$ such that the condition number of $D_k$, denoted by $\kappa(D_k)$, satisfies $\kappa(D_k)\leq \kappa_D$ for all $k \leq T$.
\end{assumption}

We justify this assumption from several perspectives. First, imposing \textit{uniform} upper and lower bound assumptions on Hessian estimates is standard in stochastic second-order optimization \citep{Berahas2021Sequential,Fang2024Trust}. Here we relax this requirement: the bounds $m_{D,k}$ and $M_{D,k}$ may be stochastic and iteration-dependent, requiring only that the condition number be deterministically bounded. Second, since $D_k$ serves as a preconditioner, it should not be too ill-conditioned, as this would distort the information contained in the gradient estimates. Third, since the algorithm runs for a finite number of iterations, the condition number should be finite almost surely. Finally, the condition number $\kappa_D$ is used only in the theoretical analysis and does not appear in the algorithm design or parameter selection.

We are now ready to present the convergence properties of Algorithm~\ref{alg:nsgdm} under heavy-tailed noise. Proofs in this section are deferred to Appendix~\ref{append_C}.

\subsection{When all algorithmic parameters are known}\label{subsec:algo_para_known}

In this subsection, we provide the convergence rate under an ideal setting where all algorithmic parameters are known. This condition is commonly assumed implicitly in the optimization literature, where algorithmic parameters are used to tune  hyperparameters.

\begin{theorem}\label{thm:full-main}
Under Assumptions~\ref{assumption1}, \ref{assumption2}, \ref{assumption3}, and \ref{assumption4}, let $\Delta \coloneqq f(\bx_1) - f_*$. For any $T \geq 1$, we select $\eta= 
    \sqrt{ \frac{(1 - \theta) \Delta}{L T}}$, $\theta = 1 - \min \left\{ 1, \max \left\{\left( \frac{\Delta L }{\sigma^2 T}\right)^{\frac{p}{3p-2}},T^{-\frac{p}{2p-1}}\right\} \right\}$, then 
Algorithm~\ref{alg:nsgdm} guarantees
\begin{align}\label{thm1:main-bound}
\frac{1}{T}  \sum_{k=1}^T \mathbb{E}  \bigl[ \|\nabla f(\bx_k)\| \bigr]
& = \mathcal{O} \Biggl( 
    \frac{ (\Delta L)^{\frac{1}{2}}}{T^{\frac{1}{2}}} + \frac{ \sigma }{T^{\frac{p-1}{2p-1}}} \notag \\
& \qquad \quad + \frac{ (\Delta L)^{\frac{p-1}{3p-2}} \sigma^{\frac{p}{3p-2}} }{T^{\frac{p-1}{3p-2}}}
\Biggr).
\end{align}
\end{theorem}

This result generalizes that of \cite{Liu2025Nonconvex} by incorporating the stochastic preconditioning matrix $D_k$. Notably, all hyperparameter settings remain identical to those for NSGD. The algorithm achieves the rate $\mathcal{O}\bigl((\Delta L)^{\frac{p-1}{3p-2}} \sigma^{\frac{p}{3p-2}}/T^{\frac{p-1}{3p-2}}\bigr)$, which matches the optimal rate for SGD under heavy-tailed noise established in \citet{Liu2025Nonconvex}. When $p=2$, this reduces to $\mathcal{O}\bigl((\Delta L)^{1/4} \sigma^{1/2}/T^{1/4}\bigr)$, which is known to be minimax optimal in all algorithmic parameters ($T$, $\Delta$, $L$, and $\sigma$) \citep{arjevani2023lower}. 
We note that our dependence on $\sigma$ is $\mathcal{O}(\sigma^{\frac{p}{3p-2}})$, while \citet{Sun2025Revisiting} achieve $\mathcal{O}(\sigma^{\frac{2p-2}{3p-2}})$. Consequently, our bound is tighter when $\sigma<1$, while their bound becomes favorable for $\sigma>1$. This difference stems from our weaker smoothness assumption. \citet{Sun2025Revisiting} impose an individual Lipschitz condition requiring the gradient estimate to be bounded at initialization, which implicitly assumes light-tailed noise.

\subsection{When all algorithmic parameters are unknown}\label{subsec:algo_para_unknown}

We now provide the convergence rate when algorithmic parameters are unknown. This condition reflects many practical scenarios where problem-dependent parameters are intractable or prohibitively difficult to estimate. 

\begin{theorem}\label{thm:none-main}
Under Assumptions~\ref{assumption1}, \ref{assumption2}, \ref{assumption3}, and \ref{assumption4}, let $\Delta \coloneqq f(\bx_1) - f_*$. For any $T \geq 1$, we select $\eta= 
    T^{-\frac{3}{4}}$, $\theta = 1 - T^{-\frac{1}{2}}$, then 
Algorithm~\ref{alg:nsgdm} guarantees
\begin{align}\label{thm1:main-bound}
\frac{1}{T}  \sum_{k=1}^T \mathbb{E} & \bigl[ \|\nabla f(\bx_k)\| \bigr]
 = \mathcal{O} \Biggl( 
    \frac{\Delta + L}{T^{\frac{1}{4}}}  + \frac{ \sigma }{T^{\frac{p-1}{2p}}}
\Biggr).
\end{align}
\end{theorem}

This result matches those of \citet{Hubler2024Gradient,Liu2025Nonconvex} and is known to be tight when $D_k=I$.

\begin{remark}
Our analysis adopts a \textit{worst-case} perspective, imposing minimal assumptions on the structure of the preconditioner $D_k$. The primary objective is to elucidate how the statistical dependence between $D_k$ and the momentum/gradient estimator $m_k$ impacts the stability of step clipping versus step normalization. Consequently, the convergence rates derived here serve as unified upper bounds for a broad class of algorithms presented in Section~\ref{sec2:preconditioned_methods}. It is therefore unsurprising that our rates match those of standard NSGD without demonstrating theoretical acceleration; proving such improvements would require exploiting algorithm-specific structures to derive tighter bounds, which is beyond the scope of this unified framework.
\end{remark}

\begin{remark}
The convergence rates in Theorems~\ref{thm:full-main} and~\ref{thm:none-main} are tight in the sense that we cannot obtain better rates without additional structural assumptions on the stochastic preconditioner $D_k$. This follows immediately by observing that setting $D_k=I$ recovers NSGD, for which the convergence rates have been shown to be minimax optimal by \cite{Liu2025Nonconvex} and \cite{Hubler2024Gradient} for parameter-known and parameter-unknown scenarios, respectively.
\end{remark}

\subsection{Step normalization is robust}

We establish a key robustness property of step normalization: when the gradient estimate is stochastically preconditioned, the accumulated gradient norms remain controlled by a factor of $\sqrt{\kappa_D}$ relative to the bounds obtained in normalized SGD (NSGD). Formally, we have the following result.
\begin{lemma}\label{lemma:normalization1}
Under Assumptions \ref{assumption1}, \ref{assumption2} and \ref{assumption4}, for all $T\geq 1$, let $\Delta \coloneqq f(\bx_1) - f_*$, we have\vskip-0.5cm
\begin{equation*}
\sum_{k=1}^T\mathbb{E}\left[\|\nabla f(\bx_k)\|\right] 
\leq \sqrt{\kappa_D} \cdot \left[\frac{\Delta}{\eta}   + 2  \sum_{k=1}^T\mathbb{E}\left[\|\varepsilon_k\|\right] + \frac{\eta L T}{2}\right],
\end{equation*}\vskip-0.2cm
\noindent where $\varepsilon_k=m_k-\nabla f(\bx_k)$.
\end{lemma}

This bound is crucial: even when the stochastic preconditioner $D_k$ is poorly designed, step normalization limits performance degradation to a multiplicative factor of at most $\sqrt{\kappa_D}$ compared with standard NSGD, without introducing any new error terms. This controlled degradation explains the robustness of normalized updates in stochastically preconditioned settings.

That said, the analysis above is worst-case and assumes no additional structure on $D_k$. When the preconditioners are well-constructed (e.g., approximating a reasonable inverse Hessian or covariance structure), algorithm-specific analyses can typically establish substantially tighter bounds than the one presented here.

\subsection{Vector-valued Burkholder-type inequality}

Given the result of Lemma~\ref{lemma:normalization1}, to complete the analysis, we must bound $\sum_{k=1}^T\mathbb{E}\left[\|\varepsilon_k\|\right]$.

By the momentum recursion, we have
\begin{equation*}
   \varepsilon_k =  \theta^k \varepsilon_0 + \sum_{t=1}^k\theta^{k-t} s_t + (1-\theta) \sum_{t=1}^k\theta^{k-t} \delta_t,
\end{equation*}
where $s_t = \nabla f(\bx_{t-1}) - \nabla f(\bx_{t})$ and $\delta_t = \barg_t - \nabla f(\bx_t)$. Thus
\begin{multline*}
   \mathbb{E}[\| \varepsilon_k\| ] \leq   \theta^k \mathbb{E}[ \| \varepsilon_0\| ] \\
   + \sum_{t=1}^k\theta^{k-t}\mathbb{E}[ \| s_t\|] + (1-\theta) \mathbb{E}\left[ \left\|\sum_{t=1}^k\theta^{k-t} \delta_t\right\|\right].
\end{multline*}
The first two terms are straightforward to bound. However, bounding $\mathbb{E}\left[\left\|\sum_{t=1}^k\theta^{k-t} \delta_t\right\|\right]$ might introduce difficulty.
By Hölder's inequality, we have
\begin{multline*}
    \mathbb{E}\left[\left\|\sum_{t=1}^k\theta^{k-t} \delta_t\right\|\right]  = \mathbb{E}\left[\left\|\sum_{t=1}^k\theta^{k-t} (\barg_t-\nabla f(\bx_t))\right\|\right] \\
     \leq \left(\mathbb{E}\left[\left\|\sum_{t=1}^k\theta^{k-t} \left(\barg_t-\nabla f(\bx_t)\right)\right\|^p\right] \right)^{1/p}.
\end{multline*}
When $p=2$, this bound is straightforward to evaluate using the martingale property. However, for $p\in(1,2)$, further simplification is nontrivial. \citet{Liu2025Nonconvex} treat the cases $p\in(1,2)$ and $p=2$ separately using different techniques. \citet{Hubler2024Gradient} provide a unified analysis, but with suboptimal constants. To address this, we establish the following vector-valued Burkholder-type inequality, which unifies the analysis for all $p\in (1,2]$ while achieving sharper constants. The proof is deferred to Appendix~\ref{append_E}.
\begin{lemma}[Vector-valued Burkholder-type inequality]\label{Burkholder}
Let $X_k, k=0,1,\dots,T-1$, be random variables with $\mathbb{E}[X_k\mid X_{0:k-1}]=\boldsymbol{0}$ and $\mathbb{E}[\|X_k\|^{p}]<\infty$ for any $k\geq 0$ and some $p \in (1,2]$. Then 
\begin{equation*}
\mathbb{E}\left[ \left\| \sum_{k=0}^{T-1} X_k \right\| ^{p} \right] \leq 2^{2-p} \sum_{k=0}^{T-1} \mathbb{E}[\|X_k\|^{p}].
\end{equation*}
\end{lemma}	

Applying the inequality and Assumption \ref{assumption4}, we have
\begin{align*}
    & \mathbb{E}\left[\left\|\sum_{t=1}^k\theta^{k-t} \delta_t\right\|\right]\\
    & \leq 2^{\frac{2}{p}-1}\left(\sum_{t=1}^k\theta^{p(k-t)}\mathbb{E}\left[\left\|\barg_t - \nabla f(\bx_t))\right\|^p\right] \right)^{1/p} \\
    & \leq \frac{2^{\frac{2}{p}-1} \sigma}{(1-\theta^p)^{1/p}}.
\end{align*}

\section{What Might Go Wrong for Clipping?}\label{sec4:clipping}

In this section, we demonstrate the potential failure of step clipping when gradient estimates are stochastically preconditioned. To simplify the exposition, we omit momentum from the gradient estimate; this choice is discussed further in Remark~\ref{remark:momentum}.

Recall that for SGD with gradient clipping, we have
\begin{equation*}
    \bx_{k+1} = \bx_k - \eta \min\left\{1,\frac{\tau}{\|\barg_k\|}\right\}\barg_k,
\end{equation*}
where $\tau$ is a clipping threshold that depends on algorithmic parameters $T$ and $\sigma$ to achieve the optimal convergence rate \citep{Zhang2020Adaptive}.

When the gradient is stochastically preconditioned, two natural choices arise. The first clips the gradient estimate before preconditioning:
\begin{equation*}
    \bx_{k+1} = \bx_k - \eta D_k \cdot\min\left\{1,\frac{\tau}{\|\barg_k\|}\right\}\barg_k,
\end{equation*}
while the second clips the preconditioned gradient:
\begin{equation*}
    \bx_{k+1} = \bx_k - \eta \min\left\{1,\frac{\tau}{\|D_k\barg_k\|}\right\}D_k\barg_k.
\end{equation*}
We refer to these as \textit{Clipping-Then-Preconditioning} and \textit{Preconditioning-Then-Clipping}, respectively.

Both designs have intuitive justifications. The first approach clips the heavy-tailed gradient estimate to remove outliers before using the stabilized estimates to construct $D_k$, yielding more robust curvature information. The second accounts for the stochastic dependence between $D_k$ and $\barg_k$: since $D_k$ may rotate and stretch $\barg_k$ stochastically, clipping after forming the complete step may produce more stable updates. Several optimizers for large-scale machine learning, such as Sophia \citep{Liu2024Sophia},  follow the second principle.

For the analysis of clipping, we impose a slightly stronger assumption on the preconditioning matrix $D_k$.
\begin{assumption}[Uniform Boundedness of $D_k$]\label{assumption5}
    There exist deterministic constants $m_D,M_D>0$ such that $m_D\cdot I \preceq D_k \preceq M_D\cdot I$ for $\forall k \leq T$.
\end{assumption}
This assumption strengthens Assumption~\ref{assumption4} by requiring that the singular values of $D_k$ be uniformly bounded by deterministic constants for all iterations. However, even under this stronger assumption, clipping fails to guarantee convergence, as shown below.

We discuss the two clipping variants separately due to differences in analysis. Proofs are presented in Appendix \ref{append_D}.

\subsection{Clipping-Then-Preconditioning}\label{subsec:clipping1}

Under this technique, we write $\bx_{k+1} = \bx_k -\eta D_k \widehat{\barg}_k$ with $\widehat{\barg}_k = \min\{1, \tau/\|\barg_k\|\}\barg_k$. Following standard analysis, we consider the expected reduction per iteration by separating the cases $\|\nabla f(\bx_k)\| \geq \tau/2$ and $\|\nabla f(\bx_k)\| < \tau/2$.

We first consider $\|\nabla f(\bx_k)\| \geq \tau/2$.
\begin{lemma}\label{lemma42}
    Under Assumptions \ref{assumption2}, \ref{assumption3} and \ref{assumption5}, when $\|\nabla f(\bx_k)\|\geq \tau/2$, if we select $\eta<\frac{m_D}{24L}$ and 
    \begin{equation*}
        \tau = \max\left\{2,  4\cdot 3^{1/p}\cdot \sigma\left(\frac{M_D}{m_D}\right)^{\frac{p+1}{2p}}, 64 \sigma \frac{M_D}{m_D}\right\},
    \end{equation*}
    then
    we have
    \begin{align*}
    \mathbb{E}_k[ f(\bx_{k+1}) ] \le f(\bx_k) - \frac{1}{12} \eta m_D \|\nabla f(\bx_k)\|.
\end{align*}
\end{lemma}
This result shows that expected reduction is achieved when the clipping threshold $\tau$ is relatively small compared to the true gradient, matching the behavior of clipped SGD. However, the learning rate and clipping threshold now depend on $m_D$ and $M_D$, complicating practical implementation.

Next, we consider the case $\|\nabla f(\bx_k)\| < \tau/2$, where we show that reduction is not guaranteed in the worst case. 
\begin{lemma}\label{lemma:large_clip}
     Under Assumptions \ref{assumption2}, \ref{assumption3} and \ref{assumption5}, when $\|\nabla f(\bx_k)\| < \tau/2$, we have
    \begin{align*}
    \mathbb{E}_k[ f(\bx_{k+1}) ] & \le f(\bx_k) - \frac{1}{2} \eta \mathbb{E}_k[ \| D_k^{1/2}\nabla f(\bx_k)\|^2 ] \\
    & \quad - (\eta - L\eta^2 M_D) \nabla f(\bx_k)^T \mathbb{E}_k[ D_k \epsilon_k^u] \\
    & \quad + \frac{1}{2}\eta M_D  \|\epsilon_k^b\|^2   + \frac{\eta^2 L}{2} M_D^2 \mathbb{E}_k[ \|\epsilon_k^u\|^2],
\end{align*}
where $\epsilon_k^u=\widehat{\barg}_k-\mathbb{E}_k[\widehat{\barg}_k]$ and $\epsilon_k^b=\mathbb{E}_k[\widehat{\barg}_k]-\nabla f(\bx_k)$.
\end{lemma}

This result resembles that for ClipSGD \cite{Zhang2020Adaptive}, except for the additional error term $-(\eta - L\eta^2 M_D)\nabla f(\bx_k)^T\mathbb{E}_k[D_k\epsilon_k^u]$. In ClipSGD, this term vanishes because $D_k = I$ and $\mathbb{E}_k[\epsilon_k^u] = 0$. However, it generally does not vanish when $D_k$ and $\barg_k$ are dependent, as in SPSGD. Note that
\begin{equation*}
    \mathbb{E}_k[D_k\epsilon_k^u] = \mathbb{E}_k[D_k\widehat{\barg}_k] - \mathbb{E}_k[D_k]\mathbb{E}_k[\widehat{\barg}_k],
\end{equation*}
this term essentially measures the covariance between the preconditioner and the clipped gradient noise.

Without explicit knowledge of $D_k$'s structure, we can only upper bound this error term. Applying the rescaled Young's inequality yields
\begin{align}\label{eq:additional_error}
    & - (\eta - L\eta^2 M_D) \nabla f(\bx_k)^T \mathbb{E}_k[ D_k \epsilon_k^u] \notag\\
    & \qquad \leq \frac{1}{4}(\eta - L\eta^2 M_D) \mathbb{E}_k[\|D_k^{1/2}\nabla f(\bx_k)\|^2] \notag \\
    & \qquad \quad + (\eta - L\eta^2 M_D) M_D \mathbb{E}_k[ \| \epsilon_k^u\|^2].
\end{align}
The rescaling ensures that the coefficient of $\mathbb{E}_k[\|D_k^{1/2}\nabla f(\bx_k)\|^2]$ remains negative. The key issue, and the source of clipping's potential failure, is the $\mathcal{O}(\eta\mathbb{E}_k[\|\epsilon_k^u\|^2])$ error term. Combining this with Lemma~\ref{lemma:large_clip}, we obtain
\begin{align*}
    \mathbb{E}_k[ f(\bx_{k+1}) ] & \le f(\bx_k) - \frac{1}{4} \eta \mathbb{E}_k[ \| D_k^{1/2}\nabla f(\bx_k)\|^2 ] \\
    & \quad + \eta M_D \mathbb{E}_k[ \| \epsilon_k^u\|^2]+ \frac{1}{2}\eta M_D  \|\epsilon_k^b\|^2 .
\end{align*}
Comparing this with ClipSGD \citep{Zhang2020Adaptive}, the current error is $\mathcal{O}(\eta\mathbb{E}_k[\|\epsilon_k^u\|^2])$ rather than $\mathcal{O}(\eta^2\mathbb{E}_k[\|\epsilon_k^u\|^2])$. By standard analysis, when $\|\nabla f(\bx_k)\| < \tau/2$, we have $\mathbb{E}_k[\|\epsilon_k^u\|^2] \leq 10\tau^{2-p}\sigma^p$ and $\|\epsilon_k^b\|^2 \leq 4\sigma^{2p}\tau^{-2p+2}$. Substituting these bounds yields
\begin{align*}
    \mathbb{E}_k[ f(\bx_{k+1}) ] & \le f(\bx_k) - \frac{1}{4} \eta m_D \|\nabla f(\bx_k)\|^2 \\
    & \quad + 10 \eta M_D \tau^{2-p}\sigma^p\\
    & \quad + 2\eta M_D \sigma^{2p}\tau^{-2p+2}.
\end{align*}
To achieve optimal complexity, the clipping threshold must be carefully chosen. Following \citet{Zhang2020Adaptive}, we set $\tau = \mathcal{O}(T^a)$ for some $a > 0$. However, when $p \in (1,2)$, this yields $\tau^{2-p}\sigma^p \to \infty$ as $T \to \infty$, indicating that descent is not guaranteed and divergence may occur.

In contrast, for ClipSGD the error term is $\mathcal{O}(\eta^2\mathbb{E}_k[\|\epsilon_k^u\|^2])$, so we can choose $\eta$ sufficiently small such that $\eta^2\tau^{2-p}\sigma^p \leq \mathcal{O}(\eta\tau^{-2p+2}\sigma^p)$, making the power of $\tau$ negative.

\begin{remark}
One might argue that the bound in \eqref{eq:additional_error} is too loose. However, given the dependence between $D_k$ and $\barg_k$ and the lack of structural assumptions on $D_k$, this is the tightest bound we can obtain without algorithm-specific analysis. The following example demonstrates that the error bound in \eqref{eq:additional_error} can indeed be achieved.

\begin{example}\label{example}
Consider the one-dimensional case. Suppose at iterate $\bx_k$ we have $\nabla f(\bx_k) = C$, where $\sigma\gg C>0$ is a constant. The gradient estimate follows a discrete distribution taking two values with equal probability: $ \barg_k = \nabla f(\bx_k) + \sigma$ with probability $1/2$ and $\barg_k = \nabla f(\bx_k) - \sigma$ with probability $1/2$. We set $M_D-m_D=\mathcal{O}(\sigma)$, and design $D_k$ such that $D_k = m_D \cdot I$ when $\barg_k > 0$ and $D_k = M_D \cdot I$ when $\barg_k < 0$. Setting $\tau \gg \sigma$, we have
\begin{equation*}
    -(\eta - L\eta^2 M_D)\nabla f(\bx_k)^T\mathbb{E}_k[D_k\epsilon_k^u] = \mathcal{O}(\eta\mathbb{E}_k[\|\epsilon_k^u\|^2]).
\end{equation*}
The proof is deferred to Appendix~\ref{proof_example}.
\end{example}
\end{remark}

\begin{remark}\label{remark:momentum}
We do not combine clipping with momentum, consistent with most standard analyses of gradient clipping under heavy-tailed noise \citep{Zhang2020Adaptive,Puchkin2024Breaking,Sadiev2023High,nguyen2023improved}. While \citet{Sun2025Revisiting} combine clipping with normalization and momentum, such combinations obscure the individual effects of clipping and normalization. We argue that the problem with clipping cannot be resolved by introducing momentum. The issue originates from the dependence between the preconditioner and the gradient estimate, while momentum merely aggregates the randomness of $\barg_k$ across iterations without removing its dependence on $D_k$. Consequently, momentum alone cannot resolve this fundamental issue.
\end{remark}

\subsection{Preconditioning-Then-Clipping}\label{subsec:clipping2}
Under this technique, we write $\bx_{k+1} = \bx_k -\eta D_k\widehat{\barg}_k$ with $\widehat{\barg}_k = \min\{1, \tau/\|D_k\barg_k\|\}\barg_k$. Similar to Section~\ref{subsec:clipping1}, we analyze this by separating cases based on gradient magnitude. Due to the preconditioning matrix in the clipping criterion, we use the threshold $\tau/(2M_D)$, which reduces to the standard $\tau/2$ when $M_D = 1$.

When the clipping threshold is relatively small, i.e., $\|\nabla f(\bx_k)\|\geq  \tau/(2M_D)$, expected descent can still be achieved, provided that the clipping threshold is properly selected. We defer the lemma to Appendix \ref{lemmac8}. However, as in Section~\ref{subsec:clipping1}, when the threshold is large, descent is not guaranteed.
\begin{lemma}
 Under Assumptions \ref{assumption2}, \ref{assumption3} and \ref{assumption5}, 
    when $\|\nabla f(\bx_k)\|< \frac{\tau}{2M_D}$, we have \begin{align*}
      \mathbb{E}_k[ f(\bx_{k+1}) ] & \le f(\bx_k) - \frac{1}{4} \eta m_D \|\nabla f(\bx_k)\|^2 \\
      & \quad + \eta M_D \Upsilon_1^2 \sigma^{2p}\tau^{-2p+2} +  \eta M_D \Upsilon_2 \tau^{2-p}\sigma^p,
\end{align*}
where $\Upsilon_1,\Upsilon_2$ are constants independent to $\sigma$ and $\tau$, $\epsilon_k^u=\widehat{\barg}_k-\mathbb{E}_k[\widehat{\barg}_k]$, and $\epsilon_k^b=\mathbb{E}_k[\widehat{\barg}_k]-\nabla f(\bx_k)$.
\end{lemma}
The reasoning is identical to that in Section~\ref{subsec:clipping1}: the additional error term $-(\eta - L\eta^2 M_D)\nabla f(\bx_k)^T\mathbb{E}_k[D_k\epsilon_k^u]$ appears and cannot be eliminated in general. 

This is an expected result, since this additional error term captures the covariance between $D_k$ and $\barg_k$, which comes from the dependence between $D_k$ and $\barg_k$, which generally cannot be overcome by step clipping.
\section{Discussion}

\subsection{Applications to large-scale machine learning}

Our theoretical findings provide a novel explanation for the empirical success of modern large-scale optimizers such as LARS \citep{you2017large,you2018imagenet}, LAMB \citep{you2019large}, and their variants. These methods combine gradient normalization with adaptive (often layer-wise) preconditioning, mirroring the structure analyzed in this work. Since large-scale models exhibit heavy-tailed gradient noise where bounded-variance assumptions fail, our results formalize why these methods scale reliably: step normalization remains robust by removing magnitude information and avoiding dependence-induced errors under stochastic preconditioning.

\subsection{Clipping is still important}
Our analysis does not dismiss clipping but rather clarifies its differences from normalization under stochastic preconditioning. Clipping remains widely used and important in practice. Moreover, our worst-case analysis does not imply clipping always fails when $D_k$ depends on $\barg_k$—the theory-practice gap mentioned in Section~\ref{intro} makes real-world performance difficult to predict, and careful tuning may still ensure good behavior.

\subsection{Future work}
Several promising directions remain for future work. First, characterizing instance-dependent or average-case behavior under realistic distributions could reveal when clipping succeeds empirically. Second, extending the theory to layer-wise normalization and block-structured preconditioners would further bridge the theory-practice gap. Finally, while our analysis considers linear preconditioning, recent stochastic methods introduce nonlinear transformations of gradient information; extending our framework to such settings would better guide optimizer design.

\section{Conclusion}
This paper provides a worst-case complexity theory explaining a fundamental separation between step clipping and step normalization under stochastic preconditioning and heavy-tailed noise. We show that clipping may fail due to a structural, dependence-induced bias that generally cannot be eliminated by tuning or clipping, while normalization remains robust and achieves the convergence rates that match the optimal rates of NSGD. Our resulabodern large-scale optimizers such as LARS and LAMB.


\section*{Acknowledgement}
This work was in part supported by the U. S. Army Research Laboratory and the U. S. Army Research Office under Grant W911NF2010219, Office of Naval Research under Grant N000142412673, and NSF.

\bibliography{main}

\nocite{langley00}

\bibliographystyle{icml2026}

\newpage
\appendix
\onecolumn
\section{Examples of Stochastically Preconditioned Methods}\label{appendix:table}

\begin{table}[h]
\centering
\caption{Summary of Stochastically Preconditioned Methods}
\label{tab:preconditioned_methods}
\begin{threeparttable}
\begin{tabular}{lccc}
\toprule
\textbf{Method} & \textbf{Preconditioner $D_k$} & \textbf{Momentum $\theta$} & \textbf{Reference} \\
\midrule
SGD & $I$ & $0$ & \citet{Robbins1951Stochastic} \\
Momentum SGD & $I$ & $\in (0,1)$ & \citet{Polyak1964Some} \\
\midrule
AdaGrad & $\operatorname{diag}(1/\sqrt{G_k+\epsilon})$ & $0$ & \citet{Duchi2011Adaptive} \\
RMSProp & $\operatorname{diag}(1/\sqrt{v_k+\epsilon})$ & $0$ & \citet{Tieleman2012Lecture} \\
Adam & $\operatorname{diag}(1/\sqrt{v_k+\epsilon})$ & $\beta_1$ & \citet{Kingma2014Adam} \\
Adafactor & $\operatorname{diag}(1/\sqrt{v_k'+\epsilon})$ & varies & \citet{zhai2022scaling} \\
\midrule
K-FAC & $(A_k \otimes B_k)^{-1}$ & varies & \citet{martens2015optimizing} \\
Shampoo & $(L_k \otimes R_k)^{-1/4}$ & $0$ & \citet{Gupta2018Shampoo} \\
\midrule
Stochastic Newton & $(H_k + \lambda I)^{-1}$ & varies & \citet{bottou2018optimization} \\
Stochastic BFGS & $B_k^{-1}$ & varies & \citet{Wang2017Stochastic} \\
Trust Region & $(H_k + \lambda_k I)^{-1}$ & varies & \citet{Curtis2022Fully} \\
\bottomrule
\end{tabular}
\begin{tablenotes}
\small
\item Notes: $G_k = \sum_{i=0}^k \barg_i \odot \barg_i$ (AdaGrad); $v_k = \beta v_{k-1} + (1-\beta)\barg_k^2$ (RMSProp, Adam); $v_k'$ is a low-rank approximation of $v_k$ (Adafactor); $A_k, B_k$ are Kronecker factors (K-FAC); $L_k, R_k$ are second-moment estimates (Shampoo); $H_k$ is the Hessian estimate; $B_k$ is the quasi-Newton approximation; $\lambda, \lambda_k \geq 0$ are regularization parameters.
\end{tablenotes}
\end{threeparttable}
\end{table}
\section{Proof of Section \ref{sec3}}\label{append_C}

\begin{lemma}\label{lemma:append1}
    Under Assumptions \ref{assumption1}, \ref{assumption2} and \ref{assumption4}, for all $T\geq 1$, we have
    \begin{equation*}
\sum_{k=1}^T\mathbb{E}\left[\|\nabla f(\bx_k)\|\right] 
\leq \sqrt{\kappa_D} \cdot \left[\frac{f(\bx_1) - f_{*}}{\eta}   + 2  \sum_{k=1}^T\mathbb{E}\left[\|\varepsilon_k\|\right] + \frac{\eta L T}{2}\right],
\end{equation*}
where $\varepsilon_k=m_k-\nabla f(\bx_k)$.
\end{lemma}

\begin{proof}
    Since $\bx_{k+1}=\bx_k + \Delta\bx_k$, we have $ f(\bx_{k+1}) \leq f(\bx_k)+\nabla f(\bx_k)^T\Delta\bx_k + \frac{L}{2}\|\Delta\bx_k\|^2$.
Recall that $\Delta\bx_k=-\eta \frac{D_km_k}{\|D_km_k\|}$ by the algorithm design, and $\|\Delta\bx_k\|=\eta$ by the step normalization, we have
\begin{equation*}
f(\bx_{k+1}) \leq f(\bx_k) - \eta \frac{\nabla f(\bx_k)^TD_km_k}{\|D_km_k\|} + \frac{\eta^2L}{2}.
\end{equation*}
Denote $\varepsilon_k=m_k-\nabla f(\bx_k)$, then
\begin{align*}
f(\bx_{k+1}) & \leq f(\bx_k) - \eta \frac{m_k^TD_km_k}{\|D_km_k\|} + \eta \frac{\varepsilon_k^TD_km_k}{\|D_km_k\|} + \frac{\eta^2L}{2} \\
& \leq f(\bx_k) - \eta \frac{\|D_k^{1/2}m_k\|^2}{\sqrt{M_{D,k}}\|D_k^{1/2}m_k\|} + \eta \frac{\|\varepsilon_k\| \|D_km_k\| }{\|D_km_k\|} + \frac{\eta^2L}{2} \\
& = f(\bx_k) - \eta \frac{\|D_k^{1/2}m_k\|}{\sqrt{M_{D,k}}} + \eta \|\varepsilon_k\|  + \frac{\eta^2L}{2} \\
& \stackrel{(a)}{\leq} f(\bx_k) - \eta \frac{\|D_k^{1/2}\nabla f(\bx_k)\|}{\sqrt{M_{D,k}}} + \eta \frac{\|D_k^{1/2}\varepsilon_k\|}{\sqrt{M_{D,k}}} + \eta \|\varepsilon_k\|  + \frac{\eta^2L}{2} \\
& \stackrel{(b)}{\leq} f(\bx_k) - \eta \frac{\|D_k^{1/2}\nabla f(\bx_k)\|}{\sqrt{M_{D,k}}} + 2\eta \|\varepsilon_k\|  + \frac{\eta^2 L}{2} \\
& \stackrel{(c)}{\leq} f(\bx_k) - \eta \frac{\sqrt{m_{D,k}}\|\nabla f(\bx_k)\|}{\sqrt{M_{D,k}}} + 2\eta \|\varepsilon_k\|  + \frac{\eta^2 L}{2},
\end{align*}
where in (a) we use the triangle inequality $\|D_k^{1/2}m_k\|\geq \|D_k^{1/2}\nabla f(\bx_k)\| - \|D_k^{1/2}\varepsilon_k\|$, and (b) follows from $\|D_k^{1/2}\varepsilon_k\| \leq \sqrt{M_{D,k}}\|\varepsilon_k\|$, and (c) follows from $\|D_k^{1/2}\nabla f(\bx_k)\| \geq \sqrt{m_{D,k}}\|\nabla f(\bx_k)\|$.

Since $\sqrt{M_{D,k}}/\sqrt{m_{D,k}}=\kappa(D_k)\leq \kappa_D$ by the definition of condition number and Assumption \ref{assumption4}, we further have
\begin{equation*}
    f(\bx_{k+1}) \leq f(\bx_k) - \eta\frac{\|\nabla f(\bx_k)\|}{\sqrt{\kappa_D}}+2\eta\|\varepsilon_k\| + \frac{\eta^2 L}{2}.
\end{equation*}

Rearranging the terms, we have
\begin{equation*}
  \|\nabla f(\bx_k)\| \leq \sqrt{\kappa_D} \cdot \left[\frac{f(\bx_k) - f(\bx_{k+1})}{\eta} +2 \|\varepsilon_k\| + \frac{\eta L}{2}\right].  
\end{equation*}

Sum over $k=1,\dots,T$, and using Assumption \ref{assumption1} gives 
\begin{equation*}
  \sum_{k=1}^T\|\nabla f(\bx_k)\| \leq \sqrt{\kappa_D} \cdot \left[\frac{f(\bx_1) - f_{*}}{\eta} +2 \sum_{k=1}^T\|\varepsilon_k\| + \frac{\eta L T}{2}\right].  
\end{equation*}
We complete the proof by taking total expectation on both sides.
\end{proof}

\begin{lemma}\label{lemma:append2}
Under Assumptions \ref{assumption2} and \ref{assumption3}, we have
    \begin{equation*}
\mathbb{E}\left[\|\varepsilon_k\|\right] \leq \theta^k \sigma
   + \frac{\theta\eta L}{1-\theta} +  \frac{2^{\frac{2}{p}-1}(1-\theta) \sigma}{(1-\theta^p)^{1/p}}.
    \end{equation*}
\end{lemma}

\begin{proof}
Denote $s_t = \nabla f(\bx_{t-1}) - \nabla f(\bx_{t})$ and $\delta_t = \barg_t - \nabla f(\bx_t)$, then by the momentum recursion, we have
\begin{equation*}
   \varepsilon_k =  \theta^k \varepsilon_0 + \sum_{t=1}^k\theta^{k-t} s_t + (1-\theta) \sum_{t=1}^k\theta^{k-t} \delta_t.
\end{equation*}
Taking total expectation yields    
\begin{equation*}
\mathbb{E}\left[\|\varepsilon_k\|\right] \leq \theta^k \mathbb{E}[ \| \varepsilon_0\| ] 
   + \sum_{t=1}^k\theta^{k-t}\mathbb{E}[ \| s_t\|] + (1-\theta) \mathbb{E}\left[ \left\|\sum_{t=1}^k\theta^{k-t} \delta_t\right\|\right].
    \end{equation*}

By Assumption~\ref{assumption4} and Hölder's inequality, together with the fact that $\bx_0=\bx_1$, we have
\begin{align*}
   \mathbb{E}[\|\varepsilon_0\|] = \mathbb{E}[\|\barg_1-\nabla f(\bx_1)\|]  \leq \left(\mathbb{E}[\|\barg_1 -\nabla f(\bx_1)\|^p] \right)^{1/p} \leq \sigma.
\end{align*}
In addition, by the Lipschitz continuity assumption (cf. Assumption \ref{assumption2}), we have $\|s_t\|\leq L\|\bx_k-\bx_{k-1}\|=L\eta$. Therefore,
\begin{align*}
    \sum_{t=1}^k\theta^{k-t}\mathbb{E}[ \| s_t\|]  \leq  \eta L \sum_{t=1}^k\theta^{k-t} = \frac{\theta \eta L}{1-\theta} .
\end{align*}
Next, we need to bound $\mathbb{E}\left[\left\|\sum_{t=1}^k\theta^{k-t} \delta_t\right\|\right]$.
By Hölder's inequality, we have
\begin{equation*}
    \mathbb{E}\left[\left\|\sum_{t=1}^k\theta^{k-t} \delta_t\right\|\right]  = \mathbb{E}\left[\left\|\sum_{t=1}^k\theta^{k-t} (\barg_t-\nabla f(\bx_t))\right\|\right] 
     \leq \left(\mathbb{E}\left[\left\|\sum_{t=1}^k\theta^{k-t} \left(\barg_t-\nabla f(\bx_t)\right)\right\|^p\right] \right)^{1/p}.
\end{equation*}
Next, we use vector-valued Burkholder-type inequality (cf. Lemma \ref{Burkholder}, with the proof in Appendix~\ref{append_E}).
and Assumption \ref{assumption4}, and then obtain 
\begin{align*}
     \mathbb{E}\left[\left\|\sum_{t=1}^k\theta^{k-t} \delta_t\right\|\right]  \leq 2^{\frac{2}{p}-1}\left(\sum_{t=1}^k\theta^{p(k-t)}\mathbb{E}\left[\left\|\barg_t - \nabla f(\bx_t))\right\|^p\right] \right)^{1/p}   \leq \frac{2^{\frac{2}{p}-1} \sigma}{(1-\theta^p)^{1/p}}.
\end{align*}
The proof completes by combining these bounds. 
\end{proof}

\begin{lemma}\label{lemma:append3}
Under Assumptions \ref{assumption1}, \ref{assumption2}, \ref{assumption3} and \ref{assumption4}, for all $T\geq 1$, we have
    \begin{align*}
\sum_{k=1}^T \mathbb{E}[ \|\nabla f(\bx_k)\|] & \leq \mathcal{O}\left(\frac{f(\bx_1) - f_{*}}{\eta}   + \frac{ \sigma}{1-\theta} + \sigma(1-\theta)^{1-\frac{1}{p}}T + \frac{\eta L}{1-\theta} T \right).
\end{align*}
\end{lemma}

\begin{proof}
    This result follows from the combination of Lemmas \ref{lemma:append1} and \ref{lemma:append2}, and the fact that $\frac{1-\theta}{(1-\theta^p)^{1/p}}\leq (1-\theta)^{1-\frac{1}{p}}$.
\end{proof}

\subsection{Proof of Theorem \ref{thm:full-main}}
Given the conclusion of Lemma \ref{lemma:append3}, when plugging in $\eta=\sqrt{\frac{(1-\theta)\Delta}{LT}}$, we have
\begin{align*}
\sum_{k=1}^T \mathbb{E}[ \|\nabla f(\bx_k)\|] = \mathcal{O}\left(\sqrt{\frac{\Delta LT }{1-\theta}}   + \frac{ \sigma}{1-\theta} + \sigma(1-\theta)^{1-\frac{1}{p}}T \right).
\end{align*}
Since $\theta = 1 - \min \left\{ 1, \max \left\{\left( \frac{\Delta L }{\sigma^2 T}\right)^{\frac{p}{3p-2}},T^{-\frac{p}{2p-1}}\right\} \right\}$, if letting $U=T^{-\frac{p}{2p-1}}$ and $V=\left( \frac{\Delta L }{\sigma^2 T}\right)^{\frac{p}{3p-2}}$, we have
\begin{align*}
& \frac{\sigma}{1-\theta} + \sqrt{\frac{\Delta LT}{1-\theta}} + \sigma(1-\theta)^{\frac{p-1}{p}} T 
\\
& \leq \sigma \left(1 + \frac{1}{\max\{U, V\}}\right) + \sqrt{\Delta LT\left(1 + \frac{1}{\max\{U, V\}}\right)} + \sigma T(\max\{U, V\})^{\frac{p-1}{p}}\\
& \leq \sigma\left(1 + \frac{1}{\overline{U}}\right) + \sqrt{\Delta LT\left(1 + \frac{1}{{V}}\right)} + \sigma T\left(U^{\frac{p-1}{p}} + V^{\frac{p-1}{p}}\right)
\\
& \leq \sigma + \sqrt{\Delta LT} + \frac{\sigma}{{U}} + \sqrt{\frac{\Delta LT}{{V}}} + \sigma T\left(U^{\frac{p-1}{p}} + V^{\frac{p-1}{p}}\right).
\end{align*}
We complete the proof by plugging in the definition of $U$ and $V$ and dividing both sides by $T$.

\subsection{Proof of Theorem \ref{thm:none-main}}
Given the conclusion of Lemma \ref{lemma:append3}, plugging in $\eta=T^{-\frac{3}{4}}=\sqrt{\frac{1-\theta}{T}}$ and $1-\theta = \frac{1}{\sqrt{T}}$, we get
\begin{align*}
\sum_{k=1}^T \mathbb{E}[ \|\nabla f(\bx_k)\|] \leq \mathcal{O}\left((\Delta +L) T^{3/4} + \sigma \sqrt{T} + \sigma T^{\frac{1+p}{2p}} \right).
\end{align*}
The proof is completed by dividing both sides by $T$.
\section{Proof of Section \ref{sec4:clipping}}\label{append_D}

\subsection{Clipping-Then-Preconditioning}

\begin{lemma}
    When $\|\nabla f(\bx_k)\|<\tau/2$, we have $\mathbb{E}_k[ \| \varepsilon_k^u\|^2]\leq 10\tau^{2-p}\sigma^p$ and $\mathbb{E}_k[ \| \varepsilon_k^b\|^2]\leq 4\sigma^{2p}\tau^{-2p+2}$. 
\end{lemma}
The proof can be found in standard analysis in ClipSGD, e.g., \citet{Zhang2020Adaptive,liu2023breaking}

\begin{lemma}
    Under Assumptions \ref{assumption2}, \ref{assumption3}, and \ref{assumption5}, when $\|\nabla f(\bx_k)\|<\tau/2$, then we have
    \begin{align*}
    \mathbb{E}_k[ f(\bx_{k+1}) ] & \le f(\bx_k) - \frac{1}{2} \eta \mathbb{E}_k[ \| D_k^{1/2}\nabla f(\bx_k)\|^2 ] - (\eta - L\eta^2 M_D) \nabla f(\bx_k)^T \mathbb{E}_k[ D_k \varepsilon_k^u] \\
    & \quad + \frac{1}{2}\eta M_D  \|\varepsilon_k^b\|^2  + \frac{\eta^2 L}{2} M_D^2 \mathbb{E}_k[ \|\varepsilon_k^u\|^2],
\end{align*}
\end{lemma}

\begin{proof}
We denote $\Delta\bx_k = \bx_{k+1}-\bx_k$, then 
the Taylor expansion yields
\begin{align*}
    f(\bx_{k+1})  \le f(\bx_k) + \nabla f(\bx_k)^T \Delta \bx_k + \frac{L }{2} \|\Delta \bx_k\|^2  = f(\bx_k) - \eta  \nabla f(\bx_k)^T D_k\widehat{\barg}_k + \frac{\eta^2 L }{2} \|D_k\widehat{\barg}_k\|^2.
\end{align*}
Letting $\varepsilon_k = \widehat{\barg}_k - \nabla f(\bx_k)$ and recalling $\|D_k\|\leq M_D$, we have
\begin{align*}
    f(\bx_{k+1}) \le f(\bx_k) - \left(\eta-\frac{\eta^2 L }{2} \right) \| D_k^{1/2}\nabla f(\bx_k)\|^2  - (\eta - L \eta^2 M_D) \nabla f(\bx_k)^T D_k \varepsilon_k  + \frac{\eta^2 L M_D^2}{2}  \|\varepsilon_k\|^2.
\end{align*}
We then follow standard analysis and decompose $\varepsilon_k = \varepsilon_k^u+\varepsilon_k^b$, where $\varepsilon_k^u=\widehat{\barg}_k-\mathbb{E}_k[\widehat{\barg}_k]$ and $\varepsilon_k^b=\mathbb{E}_k[\widehat{\barg}_k]-\nabla f(\bx_k)$. Thus,
\begin{align*}
    f(\bx_{k+1}) & \le f(\bx_k) - \left(\eta-\frac{\eta^2 L }{2} \right) \| D_k^{1/2}\nabla f(\bx_k)\|^2  - (\eta - L \eta^2 M_D) \nabla f(\bx_k)^T D_k \varepsilon_k^u \\
    & \quad - (\eta - L \eta^2 M_D) \nabla f(\bx_k)^T D_k \varepsilon_k^b  + \frac{\eta^2 L }{2} M_D^2 \|\varepsilon_k^u+\varepsilon_k^b\|^2.
\end{align*}
By Cauchy-Schwaz inequality and Young's inequality, we further have
\begin{align*}
   -  \nabla f(\bx_k)^T D_k \varepsilon_k^b \leq \frac{1}{2}\|D_k^{1/2}\nabla f(\bx_k)\|^2 + \frac{1}{2}\|D_k^{1/2} \varepsilon_k^b\|^2, 
\end{align*}
thus
\begin{align*}
    f(\bx_{k+1}) & \le f(\bx_k) - \left(\eta-\frac{\eta^2 L }{2} \right) \| D_k^{1/2}\nabla f(\bx_k)\|^2  - (\eta - L \eta^2 M_D) \nabla f(\bx_k)^T D_k \varepsilon_k^u \\
    & \quad + \frac{1}{2} (\eta - L \eta^2 M_D) \|D_k^{1/2}\nabla f(\bx_k)\|^2 + \frac{1}{2}(\eta - L \eta^2 M_D) \|D_k^{1/2} \varepsilon_k^b\|^2 + \frac{\eta^2 L }{2} M_D^2 \|\varepsilon_k^u+\varepsilon_k^b\|^2.
\end{align*}
Now we take expectation conditional on $\mathcal{F}_{k-1}$ and obtain
\begin{align*}
    \mathbb{E}_k[ f(\bx_{k+1}) ] & \le f(\bx_k) - \frac{1}{2} \eta \mathbb{E}_k[ \| D_k^{1/2}\nabla f(\bx_k)\|^2 ] - (\eta - L \eta^2 M_D) \nabla f(\bx_k)^T \mathbb{E}_k[ D_k \varepsilon_k^u] \\
    & \quad + \frac{1}{2}\eta M_D  \|\varepsilon_k^b\|^2  + \frac{\eta^2 L }{2} M_D^2 \mathbb{E}_k[ \|\varepsilon_k^u\|^2],
\end{align*}
where we have implicitly assumed that $M_D>1$. We also used the fact that $\mathbb{E}[\varepsilon_k^u]=0$ and $\varepsilon_k^b$ is deterministic conditional on $\bx_k$.

\end{proof}

\begin{lemma}
     Under Assumptions \ref{assumption2}, \ref{assumption3}, and \ref{assumption5}, when $\|\nabla f(\bx_k)\|<\tau/2$, without knowing more structural information of $D_k$, we can only ensure
    \begin{align*}
    \mathbb{E}_k[ f(\bx_{k+1}) ] \le f(\bx_k) - \frac{1}{4} \eta m_D \|\nabla f(\bx_k)\|^2 + 10 \eta M_D \tau^{2-p}\sigma^p + 2\eta M_D \sigma^{2p}\tau^{-2p+2}.
\end{align*}
\end{lemma}

\begin{proof}
    
Without more structural information of $D_k$, we can only derive the following upper-bound by applying the rescaled Hölder's inequality and obtain
\begin{align*}
     - (\eta - L \eta^2 M_D) \nabla f(\bx_k)^T \mathbb{E}_k[ D_k \varepsilon_k^u] \leq \frac{1}{4}(\eta - L \eta^2 M_D) \mathbb{E}_k[\|D_k^{1/2}\nabla f(\bx_k)\|^2]  + (\eta - L \eta^2 M_D) M_D \mathbb{E}_k[ \| \varepsilon_k^u\|^2].
\end{align*}
Therefore, we have
\begin{align*}
    \mathbb{E}_k[ f(\bx_{k+1}) ] & \le f(\bx_k) - \frac{1}{4} \eta \mathbb{E}_k[ \| D_k^{1/2}\nabla f(\bx_k)\|^2 ] + \eta M_D \mathbb{E}_k[ \| \varepsilon_k^u\|^2]  + \frac{1}{2}\eta M_D  \|\varepsilon_k^b\|^2\\
    & \le f(\bx_k) - \frac{1}{4} \eta m_D \mathbb{E}_k[ \| \nabla f(\bx_k)\|^2 ] + \eta M_D \mathbb{E}_k[ \| \varepsilon_k^u\|^2]  + \frac{1}{2}\eta M_D  \|\varepsilon_k^b\|^2 .
\end{align*}

When $\|\nabla f(\bx_k)\|<\tau/2$, we have $\mathbb{E}_k[ \| \varepsilon_k^u\|^2]\leq 10\tau^{2-p}\sigma^p$ and $ \|\varepsilon_k^b\|^2 \leq 4\sigma^{2p}\tau^{-2p+2}$. Plugging in them to the above inequality yields
\begin{align*}
    \mathbb{E}_k[ f(\bx_{k+1}) ] \le f(\bx_k) - \frac{1}{4} \eta m_D \|\nabla f(\bx_k)\|^2 + 10 \eta M_D \tau^{2-p}\sigma^p + 2\eta M_D \sigma^{2p}\tau^{-2p+2}.
\end{align*}
We thus complete the proof.
\end{proof}

Next we consider the case when $\|\nabla f(\bx_k)\|\geq \tau/2$.

\begin{lemma}
    Under Assumptions \ref{assumption2}, \ref{assumption3}, and \ref{assumption5}, when $\|\nabla f(\bx_k)\|\geq \tau/2$, if we select $\eta<\frac{m_D}{24L }$ and 
    \begin{equation*}
        \tau = \max\left\{2,  4\cdot 3^{1/p}\cdot \sigma\left(\frac{M_D}{m_D}\right)^{\frac{p+1}{2p}}, 64 \sigma \frac{M_D}{m_D} \right\},
    \end{equation*}
    we have
    \begin{align*}
    \mathbb{E}_k[ f(\bx_{k+1}) ] \le f(\bx_k) - \frac{1}{12} \eta m_D \|\nabla f(\bx_k)\|.
\end{align*}
\end{lemma}

\begin{proof}

By the Taylor expansion, we have
    \begin{align*}
    f(\bx_{k+1})  \le f(\bx_k) + \nabla f(\bx_k)^T \Delta \bx_k + \frac{L }{2} \|\Delta \bx_k\|^2  = f(\bx_k) - \eta  \nabla f(\bx_k)^T D_k\widehat{\barg}_k + \frac{\eta^2 L }{2} \|D_k\widehat{\barg}_k\|^2.
\end{align*}
We first consider the conditional expectation of the 
term $\nabla f(\bx_k)^T D_k\widehat{\barg}_k$. We have
\begin{align*}
\mathbb{E}_k\left[\nabla f(x_k)^T D_k \widehat{\barg}_k \right] = \mathbb{E}_k\left[\nabla f(x_k)^T D_k \bar{g}_k \cdot \boldsymbol{1}_{(\|\bar{g}_k\| \le \tau)}\right]  + \tau \cdot \mathbb{E}_k\left[\nabla f(x_k)^T D_k \frac{\bar{g}_k}{\|\bar{g}_k\|} \cdot \boldsymbol{1}_{(\|\bar{g}_k\| > \tau)} \right]
\end{align*}

We focusing on the first term:
\begin{align*}
& \mathbb{E}_k\left[\nabla f(x_k)^T D_k \bar{g}_k \cdot \boldsymbol{1}_{(\|\bar{g}_k\| \le \tau)}\right] \\
&= \mathbb{E}_k\left[\nabla f(x_k)^T D_k (\nabla f(x_k) + \bar{g}_k - \nabla f(x_k)) \cdot \boldsymbol{1}_{(\|\bar{g}_k\| \le \tau)} \right] \\
&= \mathbb{E}_k\left[\nabla f(x_k)^T D_k \nabla f(x_k) + \nabla f(x_k)^T D_k (\bar{g}_k - \nabla f(x_k)) \cdot \boldsymbol{1}_{(\|\bar{g}_k\| \le \tau)} \right]\\
& = \mathbb{E}_k\left[\nabla f(x_k)^T D_k \nabla f(x_k)\right] + \mathbb{E}_k\left[\nabla f(x_k)^T D_k (\bar{g}_k - \nabla f(x_k)) \cdot \boldsymbol{1}_{(\|\bar{g}_k\| \le \tau)}\right]\\
& \geq \mathbb{E}_k\left[\nabla f(x_k)^T D_k \nabla f(x_k)\right] - \mathbb{E}_k\left[\|D_k^{1/2}\nabla f(x_k)\| \|D_k^{1/2} (\bar{g}_k - \nabla f(x_k))\| \cdot \boldsymbol{1}_{(\|\bar{g}_k\| \le \tau)} \right]\\
& = \mathbb{E}_k\left[\nabla f(x_k)^T D_k \nabla f(x_k)\right] \\
& \quad - \mathbb{E}_k\left[\|D_k^{1/2}\nabla f(x_k)\| \|D_k^{1/2} (\bar{g}_k - \nabla f(x_k))\| \cdot \boldsymbol{1}_{(\|\bar{g}_k\| \le \tau,\|D_k^{1/2}(\bar{g}_k - \nabla f(x_k))\|\leq \sqrt{m_D}\cdot\tau/4)}\right]\\
& \quad - \mathbb{E}_k\left[\|D_k^{1/2}\nabla f(x_k)\| \|D_k^{1/2} (\bar{g}_k - \nabla f(x_k))\| \cdot \boldsymbol{1}_{(\|\bar{g}_k\| \le \tau,\|D_k^{1/2}(\bar{g}_k - \nabla f(x_k))\|> \sqrt{m_D}\cdot\tau/4)}\right] .
\end{align*}

We first consider the case when $\|\bar{g}_k\| \le \tau,\|D_k^{1/2}(\bar{g}_k - \nabla f(x_k))\|\leq \sqrt{m_D}\cdot\tau/4$.

Since $\|\nabla f(\bx_k)\|>\tau/2$, we have
\begin{equation*}
    \|D_k^{1/2}\nabla f(\bx_k)\|^2 = \nabla f(\bx_k)^TD_k\nabla f(\bx_k)\geq m_D\|\nabla f(\bx_k)\|^2\geq \frac{\tau^2m_D}{4},
\end{equation*}
which imply
\begin{equation*}
    \|D_k^{1/2}\nabla f(\bx_k)\| \geq \frac{\tau \sqrt{m_D}}{2}, \quad \|D_k^{1/2}(\bar{g}_k - \nabla f(x_k))\|\leq \sqrt{m_D}\cdot\tau/4 \leq \frac{1}{2}\|D_k^{1/2}\nabla f(\bx_k)\|.
\end{equation*}
Therefore we find
\begin{equation*}
    \|D_k^{1/2}\nabla f(x_k)\| \|D_k^{1/2} (\bar{g}_k - \nabla f(x_k))\|\leq \frac{1}{2}\|D_k^{1/2}\nabla f(\bx_k)\|^2.
\end{equation*}
We further have
\begin{align*}
& \mathbb{E}_k\left[\|D_k^{1/2}\nabla f(x_k)\| \|D_k^{1/2} (\bar{g}_k - \nabla f(x_k))\| \cdot \boldsymbol{1}_{(\|\bar{g}_k\| \le \tau,\|D_k^{1/2}(\bar{g}_k - \nabla f(x_k))\|\leq \sqrt{m_D}\cdot\tau/4)}\right] \\
& \leq \frac{1}{2} \mathbb{E}_k\left[\|D_k^{1/2}\nabla f(\bx_k)\|^2 \cdot \boldsymbol{1}_{(\|\bar{g}_k\| \le \tau,\|D_k^{1/2}(\bar{g}_k - \nabla f(x_k))\|\leq \sqrt{m_D}\cdot\tau/4)}\right] \\
& \leq \frac{1}{2} \mathbb{E}_k\left[\|D_k^{1/2}\nabla f(\bx_k)\|^2 \cdot \boldsymbol{1}_{(\|\bar{g}_k\| \le \tau)}\right]
\end{align*}
Plugging this into the above formula, we have
\begin{align*}
 & \mathbb{E}_k\left[\nabla f(x_k)^T D_k \bar{g}_k \cdot \boldsymbol{1}_{(\|\bar{g}_k\| \le \tau)}\right] \\
& \geq \frac{1}{2}\mathbb{E}_k\left[\| D_k^{1/2} \nabla f(x_k)\|^2\cdot \boldsymbol{1}_{(\|\bar{g}_k\| \le \tau)}\right] - \mathbb{E}_k\left[\|D_k^{1/2}\nabla f(x_k)\| \|D_k^{1/2} (\bar{g}_k - \nabla f(x_k))\| \cdot \boldsymbol{1}_{(\|D_k^{1/2}(\bar{g}_k - \nabla f(x_k))\|> \sqrt{m_D}\cdot\tau/4)}\right]\\
& \geq \frac{1}{2}\mathbb{E}_k\left[\| D_k^{1/2} \nabla f(x_k)\|^2\right]\cdot \mathbb{E}_k\left[\boldsymbol{1}_{(\|\bar{g}_k\| \le \tau)}\right] - \sqrt{M_D}\|\nabla f(x_k)\| \mathbb{E}_k[\|D_k^{1/2} (\bar{g}_k - \nabla f(x_k))\| \cdot \boldsymbol{1}_{(\|D_k^{1/2}(\bar{g}_k - \nabla f(x_k))\|> \sqrt{m_D}\cdot\tau/4)}]\\
& = \frac{p_g}{2}\mathbb{E}_k[\| D_k^{1/2} \nabla f(x_k)\|^2] - \sqrt{M_D}\|\nabla f(x_k)\| \mathbb{E}_k[\|D_k^{1/2} (\bar{g}_k - \nabla f(x_k))\| \cdot \boldsymbol{1}_{(\|D_k^{1/2}(\bar{g}_k - \nabla f(x_k))\|> \sqrt{m_D}\cdot\tau/4)}].
\end{align*}

Next, we bound $\mathbb{E}_k[\|D_k^{1/2} (\bar{g}_k - \nabla f(x_k))\| \cdot \boldsymbol{1}_{(\|D_k^{1/2}(\bar{g}_k - \nabla f(x_k))\|> \sqrt{m_D}\cdot\tau/4)}]$.

\begin{align*}
    M_D^{p/2} \sigma^p & \geq M_D^{p/2}\mathbb{E}_k[\|\barg_k-\nabla f(\bx_k)\|^p]\geq \mathbb{E}_k[\|D_k^{1/2}(\barg_k-\nabla f(\bx_k))\|^p] \\
    & = \mathbb{E}_k[\|D_k^{1/2}(\barg_k-\nabla f(\bx_k))\| \cdot \|D_k^{1/2}(\barg_k-\nabla f(\bx_k))\|^{p-1}] \\
    & \geq \mathbb{E}_k[\|D_k^{1/2}(\barg_k-\nabla f(\bx_k))\| \cdot \|D_k^{1/2}(\barg_k-\nabla f(\bx_k))\|^{p-1}\cdot \boldsymbol{1}_{(\|D_k^{1/2}(\bar{g}_k - \nabla f(x_k))\|> \sqrt{m_D}\cdot\tau/4)}] \\
    & \geq \mathbb{E}_k[\|D_k^{1/2}(\barg_k-\nabla f(\bx_k))\| \cdot \left(\sqrt{m_D}\cdot\frac{\tau}{4}\right)^{p-1}\cdot \boldsymbol{1}_{(\|D_k^{1/2}(\bar{g}_k - \nabla f(x_k))\|> \sqrt{m_D}\cdot\tau/4)}].
\end{align*}
Rearranging the terms yields
\begin{align*}
    \mathbb{E}_k[\|D_k^{1/2}(\barg_k-\nabla f(\bx_k))\|  \cdot \boldsymbol{1}_{(\|D_k^{1/2}(\bar{g}_k - \nabla f(x_k))\|> \sqrt{m_D}\cdot\tau/4)}] \leq \frac{M_D^{p/2} \sigma^p}{\left(\sqrt{m_D}\cdot\frac{\tau}{4}\right)^{p-1}}.
\end{align*}
Thus
\begin{align*}
    \mathbb{E}_k[\nabla f(x_k)^T D_k \bar{g}_k \cdot \boldsymbol{1}_{(\|\bar{g}_k\| \le \tau)}]
    \geq \frac{p_g}{2}\mathbb{E}_k[\| D_k^{1/2} \nabla f(x_k)\|^2] -  \|\nabla f(\bx_k)\|\frac{M_D^{p/2+1} \sigma^p}{\left(\sqrt{m_D}\cdot\frac{\tau}{4}\right)^{p-1}} .
\end{align*}

Next, we examine $\tau \cdot \mathbb{E}_k\left[\nabla f(x_k)^T D_k \frac{\bar{g}_k}{\|\bar{g}_k\|} \cdot \boldsymbol{1}_{(\|\bar{g}_k\| > \tau)} \right]$.

We first consider the case $\|D_k^{1/2}\nabla f(\bx_k)\|\geq 2\|D_k^{1/2}(\barg_k-\nabla f(\bx_k))\|$, where we have
\begin{align*}
    \nabla f(\bx_k)^TD_k\barg_k 
    &= \nabla f(\bx_k)^TD_k\nabla f(\bx_k) +\nabla f(\bx_k)^TD_k (\barg_k - \nabla f(\bx_k)) \\
    & \geq \nabla f(\bx_k)^TD_k\nabla f(\bx_k) - \|D_k^{1/2}\nabla f(\bx_k)\|\|D_k^{1/2} (\barg_k - \nabla f(\bx_k))\| \\
    & \geq \nabla f(\bx_k)^TD_k\nabla f(\bx_k) - \frac{1}{2}\|D_k^{1/2}\nabla f(\bx_k)\|^2 \\
    & = \frac{1}{2}\|D_k^{1/2}\nabla f(\bx_k)\|^2.
\end{align*}

In addition, 
\begin{align*}
    \sqrt{m_D}\|\nabla f(\bx_k) + \barg_k - \nabla f(\bx_k)\| & \leq \|D_k^{1/2}(\nabla f(\bx_k) + \barg_k - \nabla f(\bx_k))\|\\
    & \leq \|D_k^{1/2}\nabla f(\bx_k)\| + \| D_k^{1/2}(\barg_k - \nabla f(\bx_k))\| \\
    & \leq \frac{3}{2}\|D_k^{1/2}\nabla f(\bx_k)\|.
\end{align*}
Rearranging the terms yields
\begin{equation*}
   \|\nabla f(\bx_k) + \barg_k - \nabla f(\bx_k)\| \leq  \frac{3}{2\sqrt{m_D}}\|D_k^{1/2}\nabla f(\bx_k)\|.
\end{equation*}

Therefore, we have
\begin{align*}
    \nabla f(\bx_k)^TD_k\frac{\barg_k}{\|\barg_k\|} & = \frac{\nabla f(\bx_k)^TD_k(\nabla f(\bx_k)+\barg_k-\nabla f(\bx_k))}{\|\nabla f(\bx_k)+\barg_k-\nabla f(\bx_k)\|}\\
    & \geq \frac{\sqrt{m_D}}{3}\|D_k^{1/2}\nabla f(\bx_k)\|.
\end{align*}

When $\|D_k^{1/2}\nabla f(\bx_k)\|<2\|D_k^{1/2}(\barg_k-\nabla f(\bx_k))\|$, we have
\begin{align*}
    \nabla f(\bx_k)^TD_k\frac{\barg_k}{\|\barg_k\|} & \geq - \| D_k^{1/2}\nabla f(\bx_k)\| \frac{\| D_k^{1/2}\barg_k\|}{\|\barg_k\|}\\
   & \geq - \sqrt{M_D}\| D_k^{1/2}\nabla f(\bx_k)\|\\
   & =  \frac{\sqrt{m_D}}{3}\| D_k^{1/2}\nabla f(\bx_k)\| - \left( \frac{\sqrt{m_D}}{3} + \sqrt{M_D}\right)\| D_k^{1/2}\nabla f(\bx_k)\| \\
   & \geq \frac{\sqrt{m_D}}{3}\| D_k^{1/2}\nabla f(\bx_k)\| - 2\left( \frac{\sqrt{m_D}}{3} + \sqrt{M_D}\right)\| D_k^{1/2}(\barg_k-\nabla f(\bx_k))\|.
\end{align*}
Combining the above formulas, we have
\begin{align*}
   & \tau \cdot \mathbb{E}_k\left[\nabla f(x_k)^T D_k \frac{\bar{g}_k}{\|\bar{g}_k\|} \cdot \boldsymbol{1}_{(\|\bar{g}_k\| > \tau)} \right] \\
   & \geq  
   \tau \cdot \mathbb{E}_k\left[\left(\frac{\sqrt{m_D}}{3}\| D_k^{1/2}\nabla f(\bx_k)\| - 2\left( \frac{\sqrt{m_D}}{3} + \sqrt{M_D}\right)\| D_k^{1/2}(\barg_k-\nabla f(\bx_k))\|\right) \cdot \boldsymbol{1}_{(\|\bar{g}_k\| > \tau)} \right] \\
   & = \tau \frac{\sqrt{m_D}}{3} \cdot \mathbb{E}_k\left[\| D_k^{1/2}\nabla f(\bx_k)\| \cdot \boldsymbol{1}_{(\|\bar{g}_k\| > \tau)} \right]  \\
   & \quad - 2\left( \frac{\sqrt{m_D}}{3} + \sqrt{M_D}\right) \tau \cdot \mathbb{E}_k\left[ \| D_k^{1/2}(\barg_k-\nabla f(\bx_k))\|\cdot \boldsymbol{1}_{(\|\bar{g}_k\| > \tau) } \right] \\
   & \geq \tau \frac{\sqrt{m_D}}{3} \cdot \mathbb{E}_k\left[\| D_k^{1/2}\nabla f(\bx_k)\|\right] \cdot \mathbb{E}_k\left[\boldsymbol{1}_{(\|\bar{g}_k\| > \tau) }\right]  \\
   & \quad - 2\tau \left( \frac{\sqrt{m_D}}{3} + \sqrt{M_D}\right)  \cdot \mathbb{E}_k\left[ \| D_k^{1/2}(\barg_k-\nabla f(\bx_k))\|\cdot \boldsymbol{1}_{(\|\bar{g}_k\| > \tau)} \right] \\
   & \geq \tau (1-p_g) \frac{\sqrt{m_D}}{3} \cdot \mathbb{E}_k\left[\| D_k^{1/2}\nabla f(\bx_k)\|\right]  - 2\tau \left( \frac{\sqrt{m_D}}{3} + \sqrt{M_D}\right)  \cdot \mathbb{E}_k\left[ \| D_k^{1/2}(\barg_k-\nabla f(\bx_k))\|  \right] \\
   & \geq \tau (1-p_g) \frac{\sqrt{m_D}}{3} \cdot \mathbb{E}_k\left[\| D_k^{1/2}\nabla f(\bx_k)\|\right]  - 2\tau \left( \frac{\sqrt{m_D}}{3} + \sqrt{M_D}\right) \sqrt{M_D} \cdot \sigma .
\end{align*}

Combining the above derivations, we have
\begin{align*}
    \mathbb{E}_k[\nabla f(x_k)^T D_k \widehat{\barg}_k ] & = \mathbb{E}_k[\nabla f(x_k)^T D_k \bar{g}_k \cdot \boldsymbol{1}_{(\|\bar{g}_k\| \le \tau)} ]  + \tau \cdot \mathbb{E}_k\left[\nabla f(x_k)^T D_k \frac{\bar{g}_k}{\|\bar{g}_k\|} \cdot \boldsymbol{1}_{(\|\bar{g}_k\| > \tau)} \right] \\
    & \geq \frac{p_g}{2}\mathbb{E}_k[\| D_k^{1/2} \nabla f(x_k)\|^2] -  \|\nabla f(\bx_k)\|\frac{M_D^{p/2+1} \sigma^p}{\left(\sqrt{m_D}\cdot\frac{\tau}{4}\right)^{p-1}} \\
    & \quad + \tau (1-p_g) \frac{\sqrt{m_D}}{3} \cdot \mathbb{E}_k\left[\| D_k^{1/2}\nabla f(\bx_k)\|\right]  - 2\tau \left( \frac{\sqrt{m_D}}{3} + \sqrt{M_D}\right) \sqrt{M_D} \cdot \sigma .
\end{align*}
Note that since $\|D_k^{1/2}\nabla f(\bx_k)\|\geq \sqrt{m_D}\|\nabla f(\bx_k)\|$, we have
\begin{align*}
    \mathbb{E}_k[\nabla f(x_k)^T D_k \widehat{\barg}_k ]
    & \geq \frac{p_g}{2}\sqrt{m_D}\mathbb{E}_k[\| D_k^{1/2} \nabla f(x_k)\|\|\nabla f(\bx_k)\|] -  \|\nabla f(\bx_k)\|\frac{M_D^{p/2+1} \sigma^p}{\left(\sqrt{m_D}\cdot\frac{\tau}{4}\right)^{p-1}} \\
    & \quad + \tau (1-p_g) \frac{\sqrt{m_D}}{4} \cdot \mathbb{E}_k\left[\| D_k^{1/2}\nabla f(\bx_k)\|\right]  - 2\tau \left( \frac{\sqrt{m_D}}{3} + \sqrt{M_D}\right) \sqrt{M_D} \cdot \sigma .
\end{align*}
Since $\|\nabla f(\bx_k)\|\geq \tau/2$, we further have
\begin{align*}
    \mathbb{E}_k[\nabla f(x_k)^T D_k \widehat{\barg}_k ]
    & \geq \frac{p_g}{4}\sqrt{m_D}\tau\cdot\mathbb{E}_k[\| D_k^{1/2} \nabla f(x_k)\|] -  \|\nabla f(\bx_k)\|\frac{M_D^{p/2+1} \sigma^p}{\left(\sqrt{m_D}\cdot\frac{\tau}{4}\right)^{p-1}} \\
    & \quad + \tau (1-p_g) \frac{\sqrt{m_D}}{4} \cdot \mathbb{E}_k\left[\| D_k^{1/2}\nabla f(\bx_k)\|\right]  - 2\tau \left( \frac{\sqrt{m_D}}{3} + \sqrt{M_D}\right) \sqrt{M_D} \cdot \sigma \\
    & = \tau  \frac{\sqrt{m_D}}{4} \cdot \mathbb{E}_k\left[\| D_k^{1/2}\nabla f(\bx_k)\|\right] -  \|\nabla f(\bx_k)\|\frac{M_D^{p/2+1} \sigma^p}{\left(\sqrt{m_D}\cdot\frac{\tau}{4}\right)^{p-1}} - 2\tau \left( \frac{\sqrt{m_D}}{3} + \sqrt{M_D}\right) \sqrt{M_D} \cdot \sigma \\
    & \geq \tau  \frac{m_D}{4} \cdot \| \nabla f(\bx_k)\| -  \|\nabla f(\bx_k)\|\frac{M_D^{p/2+1} \sigma^p}{\left(\sqrt{m_D}\cdot\frac{\tau}{4}\right)^{p-1}} - 2\tau \left( \frac{\sqrt{m_D}}{3} + \sqrt{M_D}\right) \sqrt{M_D} \cdot \sigma.
\end{align*}
Since we set 
\begin{equation*}
    \tau \geq \max\left\{2,  4\cdot 3^{1/p}\cdot \sigma\left(\frac{M_D}{m_D}\right)^{\frac{p+1}{2p}}, 64 \sigma \frac{M_D}{m_D} \right\},
\end{equation*}
we further have
\begin{equation*}
    \frac{M_D^{p/2+1} \sigma^p}{\left(\sqrt{m_D}\cdot\frac{\tau}{4}\right)^{p-1}} \leq \frac{\tau m_D}{12}, \quad 2\tau \left( \frac{\sqrt{m_D}}{3} + \sqrt{M_D}\right) \sqrt{M_D} \cdot \sigma \leq \frac{\tau^2}{24}m_D.
\end{equation*}

These imply
\begin{align*}
    \mathbb{E}_k[\nabla f(x_k)^T D_k \widehat{\barg}_k ]
    & \geq   \frac{\tau m_D}{4} \cdot \| \nabla f(\bx_k)\| -  \frac{\tau m_D}{12}\|\nabla f(\bx_k)\| - \frac{\tau^2}{24}m_D\\
    & \geq \frac{\tau m_D}{4} \cdot \| \nabla f(\bx_k)\| -  \frac{\tau m_D}{6}\|\nabla f(\bx_k)\|\\
    & = \frac{\tau m_D}{12}\|\nabla f(\bx_k)\|.
\end{align*}

Therefore, 
 \begin{align*}
\mathbb{E}_k[f(\bx_{k+1})]  & \leq f(\bx_k) - \eta  \mathbb{E}_k[\nabla f(\bx_k)^T D_k\widehat{\barg}_k] + \frac{\eta^2 L }{2} \mathbb{E}_k[\|D_k\widehat{\barg}_k\|^2]\\
& \leq - \eta \frac{\tau m_D}{12}\|\nabla f(\bx_k)\| + \frac{\eta^2 L }{2}\tau^2 \\
& \leq - \eta \frac{\tau m_D}{12}\|\nabla f(\bx_k)\| + \eta^2 L  \tau \|\nabla f(\bx_k)\| \\
& \leq - \eta \frac{\tau m_D}{24}\|\nabla f(\bx_k)\| \\
& \leq - \eta \frac{ m_D}{12}\|\nabla f(\bx_k)\|.
\end{align*}
We thus complete the proof.
\end{proof}

\subsection{Preconditioning-Then-Clipping}

\begin{lemma}
    When $\|\nabla f(\bx_k)\|<\frac{1}{2M_D}\tau$, we have $\| \varepsilon_k^b\|\leq \Upsilon_1\sigma^{p}\tau^{1-p}$ and $\mathbb{E}_k[ \| \varepsilon_k^u\|^2]\leq \Upsilon_2\tau^{2-p}\sigma^p$.
    where $\Upsilon_1=\frac{M_D^p}{m_D2^{1-p}}$ and $\Upsilon_2=\frac{9 M_D^p}{m_D^2} + \left(\frac{3}{2m_D}\right)^{2-p}$.
\end{lemma}
\begin{proof}
By the definition of $\varepsilon_k^b$, we have
\begin{align*}
\|\varepsilon_k^b\| & = \|\mathbb{E}[\widehat{\barg}_k-\barg_k]\|\\
& = \|\mathbb{E}[(\widehat{\barg}_k-\barg_k)\boldsymbol{1}_{(\|D_k\barg_k\|\geq \tau)}]\| + \|\mathbb{E}[(\widehat{\barg}_k-\barg_k)\boldsymbol{1}_{(\|D_k\barg_k\|< \tau)}]\| \\
& = \left\|\mathbb{E}\left[\left(\frac{\tau}{\|D_k\barg_k\|}\barg_k-\barg_k\right)\boldsymbol{1}_{(\|D_k\barg_k\|\geq \tau)}\right]\right\| \\
& \leq \mathbb{E}\left[\left\|\frac{\tau}{\|D_k\barg_k\|}\barg_k-\barg_k\right\|\boldsymbol{1}_{(\|D_k\barg_k\|\geq \tau)}\right] \\
& = \mathbb{E}\left[\|\barg_k\| \left(1- \frac{\tau}{\|D_k\barg_k\|}\right)\boldsymbol{1}_{(\|D_k\barg_k\|\geq \tau)}\right].
\end{align*}
Since $\|\barg_k\|\leq \|D_k\barg_k\|/m_D$, we have
\begin{align*}
    \|\varepsilon_k^b\| & \leq  \mathbb{E}\left[\frac{\|D_k\barg_k\|}{m_D} \left(1- \frac{\tau}{\|D_k\barg_k\|}\right)\boldsymbol{1}_{(\|D_k\barg_k\|\geq \tau)}\right] \\
    & = \mathbb{E}\left[\left(\frac{\|D_k\barg_k\|}{m_D} - \frac{\tau}{m_D}\right)\boldsymbol{1}_{(\|D_k\barg_k\|\geq \tau)}\right].
\end{align*}
Note that $\|D_k\barg_k\|\leq \|D_k(\barg_k-\nabla f(\bx_k))\|+ \|D_k\nabla f(\bx_k)\|\leq \|D_k(\barg_k-\nabla f(\bx_k))\|+ \tau$, we have
\begin{align*}
    \|\varepsilon_k^b\| & \leq  \mathbb{E}\left[\left(\frac{\|D_k\barg_k\|}{m_D} - \frac{\tau}{m_D}\right)\boldsymbol{1}_{(\|D_k\barg_k\|\geq \tau)}\right] \\
    & \leq \mathbb{E}\left[\left(\frac{\|D_k(\barg_k-\nabla f(\bx_k))\|+ \tau}{m_D} - \frac{\tau}{m_D}\right)\boldsymbol{1}_{(\|D_k\barg_k\|\geq \tau)}\right] \\
    & = \mathbb{E}\left[\frac{\|D_k(\barg_k-\nabla f(\bx_k))\|}{m_D}\boldsymbol{1}_{(\|D_k\barg_k\|\geq \tau)}\right].
\end{align*}
Note that $\|D_k\barg_k\|\geq \tau$ implies $\|D_k(\barg_k-\nabla f(\bx_k))\|\geq \tau/2$, we have
\begin{align*}
    \|\varepsilon_k^b\| & \leq \mathbb{E}\left[\frac{\|D_k(\barg_k-\nabla f(\bx_k))\|}{m_D}\boldsymbol{1}_{(\|D_k(\barg_k-\nabla f(\bx_k))\|\geq \tau/2)}\right] \\
    & =  \mathbb{E}\left[\frac{\|D_k(\barg_k-\nabla f(\bx_k))\|^p}{m_D}\|D_k(\barg_k-\nabla f(\bx_k))\|^{1-p}\boldsymbol{1}_{(\|D_k(\barg_k-\nabla f(\bx_k))\|\geq \tau/2)}\right] \\
    & \leq \mathbb{E}\left[\frac{M_D^p\|\barg_k-\nabla f(\bx_k)\|^p}{m_D}(\tau/2)^{1-p}\right] \\
    & \leq \frac{M_D^p}{m_D2^{1-p}}\sigma^p\tau^{1-p}.
\end{align*}

Now we turn to $\mathbb{E}_k[\|\varepsilon_k^u\|^2]$.
\begin{align*}
\mathbb{E}_k[\|\varepsilon_k^u\|^2] & = \mathbb{E}_k[\|\widehat{\barg}_k-\mathbb{E}_k[\widehat{\barg}_k]\|^2] \\
& \leq \mathbb{E}_k[\|\widehat{\barg}_k-\nabla f(\bx_k)\|^2] \\
& = \mathbb{E}_k[\|\widehat{\barg}_k-\nabla f(\bx_k)\|^2\cdot \boldsymbol{1}_{(\|D_k\barg_k\|>\tau)}] + \mathbb{E}_k[\|\widehat{\barg}_k-\nabla f(\bx_k)\|^2\cdot \boldsymbol{1}_{(\|D_k\barg_k\|\leq \tau)}] \\
&  = \mathbb{E}_k\left[\left\|\frac{\tau}{\|D_k\barg_k\|}\barg_k-\nabla f(\bx_k)\right\|^2\cdot \boldsymbol{1}_{(\|D_k\barg_k\|>\tau)}\right] + \mathbb{E}_k[\|\barg_k-\nabla f(\bx_k)\|^2\cdot \boldsymbol{1}_{(\|D_k\barg_k\|\leq \tau)}].
\end{align*}
We note that
\begin{align*}
    \left\|\frac{\tau}{\|D_k\barg_k\|}\barg_k-\nabla f(\bx_k)\right\| & \leq \frac{\|\tau\barg_k\|}{\|D_k\barg_k\|}+\|\nabla f(\bx_k)\| \leq \frac{\tau\|D_k\barg_k\|}{\|D_k\barg_k\|m_D}+\frac{\tau}{2M_D}  = \frac{\tau}{m_D}+\frac{\tau}{2M_D} \leq \frac{3\tau}{2m_D},
\end{align*}
and
\begin{align*}
    \|\barg_k-\nabla f(\bx_k)\|\leq \|\barg_k\|+\|\nabla f(\bx_k)\| \leq \frac{\|D_k\barg_k\|}{m_D} + \|\nabla f(\bx_k)\| \leq \frac{\tau}{m_D} + \frac{\tau}{2M_D} \leq  \frac{3\tau}{2m_D}.
\end{align*}
Therefore, we have
\begin{align*}
\mathbb{E}_k[\|\varepsilon_k^u\|^2] & \leq \mathbb{E}_k\left[\left(\frac{3\tau}{2m_D}\right)^2\cdot \boldsymbol{1}_{(\|D_k\barg_k\|>\tau)}\right] + \mathbb{E}_k\left[\left(\frac{3\tau}{2m_D}\right)^{2-p}\|\barg_k-\nabla f(\bx_k)\|^p\cdot \boldsymbol{1}_{(\|D_k\barg_k\|\leq \tau)}\right]\\
& \leq \mathbb{E}_k\left[\left(\frac{3\tau}{2m_D}\right)^2\cdot \boldsymbol{1}_{(\|D_k\barg_k\|>\tau)}\right] + \mathbb{E}_k\left[\left(\frac{3\tau}{2m_D}\right)^{2-p}\|\barg_k-\nabla f(\bx_k)\|^p\right] \\
& = \left(\frac{3\tau}{2m_D}\right)^2\cdot \mathbb{E}_k[\boldsymbol{1}_{(\|D_k\barg_k\|>\tau)}] + \left(\frac{3\tau}{2m_D}\right)^{2-p} \mathbb{E}_k[\|\barg_k-\nabla f(\bx_k)\|^p] .
\end{align*}

Since $\|D_k(\barg_k-\nabla f(\bx_k)\| \geq \|D_k\barg_k\| - \|D_k\nabla f(\bx_k)\| \geq \|D_k\barg_k\| - \|D_k\|\frac{\tau}{2M_D}\geq \|D_k\barg_k\| - \frac{\tau}{2} \geq \frac{\tau}{2}$, we have
\begin{align*}
    \mathbb{E}_k[\boldsymbol{1}_{(\|D_k\barg_k\|>\tau)}] & = \mathbb{P}_k[\|D_k\barg_k\|>\tau] \leq \mathbb{P}_k[\|D_k(\barg_k-\nabla f(\bx_k)\|>\tau/2] \leq \frac{\mathbb{E}_k[\|D_k(\barg_k-\nabla f(\bx_k)\|^p]}{(\tau/2)^p}  \leq \frac{2^p M_D^p\sigma^p}{\tau^p}
\end{align*}
Therefore, we have
\begin{align*}
\mathbb{E}_k[\|\varepsilon_k^u\|^2] & \leq \left(\frac{3\tau}{2m_D}\right)^2\frac{2^p M_D^p\sigma^p}{\tau^p} + \left(\frac{3\tau}{2m_D}\right)^{2-p} \sigma^p = \left[\frac{9 M_D^p}{m_D^2} + \left(\frac{3}{2m_D}\right)^{2-p} \right]\tau^{2-p} \sigma^p.
\end{align*}
We therefore complete the proof.
\end{proof}

\begin{lemma}
    Under Assumptions \ref{assumption2}, \ref{assumption3}, and \ref{assumption5}, when $\|\nabla f(\bx_k)\|<\frac{\tau}{2M_D}$, then we have
    \begin{align*}
    \mathbb{E}_k[ f(\bx_{k+1}) ] & \le f(\bx_k) - \frac{1}{2} \eta \mathbb{E}_k[ \| D_k^{1/2}\nabla f(\bx_k)\|^2 ] - (\eta - L \eta^2 M_D) \nabla f(\bx_k)^T \mathbb{E}_k[ D_k \varepsilon_k^u] \\
    & \quad + \frac{1}{2}\eta M_D  \|\varepsilon_k^b\|^2  + \frac{\eta^2 L }{2} M_D^2 \mathbb{E}_k[ \|\varepsilon_k^u\|^2],
\end{align*}
\end{lemma}

\begin{proof}
We denote $\Delta\bx_k = \bx_{k+1}-\bx_k$, then the Taylor expansion yields
\begin{align*}
    f(\bx_{k+1})  \le f(\bx_k) + \nabla f(\bx_k)^T \Delta \bx_k + \frac{L }{2} \|\Delta \bx_k\|^2  = f(\bx_k) - \eta  \nabla f(\bx_k)^T D_k\widehat{\barg}_k + \frac{\eta^2 L }{2} \|D_k\widehat{\barg}_k\|^2.
\end{align*}
Letting $\varepsilon_k = \widehat{\barg}_k - \nabla f(\bx_k)$ and recalling $\|D_k\|\leq M_D$, we have
\begin{align*}
    f(\bx_{k+1}) \le f(\bx_k) - \left(\eta-\frac{\eta^2 L }{2} \right) \| D_k^{1/2}\nabla f(\bx_k)\|^2  - (\eta - L \eta^2 M_D) \nabla f(\bx_k)^T D_k \varepsilon_k  + \frac{\eta^2 L M_D^2}{2}  \|\varepsilon_k\|^2.
\end{align*}
We then follow standard analysis and decompose $\varepsilon_k = \varepsilon_k^u+\varepsilon_k^b$, where $\varepsilon_k^u=\widehat{\barg}_k-\mathbb{E}_k[\widehat{\barg}_k]$ and $\varepsilon_k^b=\mathbb{E}_k[\widehat{\barg}_k]-\nabla f(\bx_k)$, then
\begin{align*}
    f(\bx_{k+1}) & \le f(\bx_k) - \left(\eta-\frac{\eta^2 L }{2} \right) \| D_k^{1/2}\nabla f(\bx_k)\|^2  - (\eta - L \eta^2 M_D) \nabla f(\bx_k)^T D_k \varepsilon_k^u \\
    & \quad - (\eta - L \eta^2 M_D) \nabla f(\bx_k)^T D_k \varepsilon_k^b  + \frac{\eta^2 L }{2} M_D^2 \|\varepsilon_k^u+\varepsilon_k^b\|^2.
\end{align*}
By Cauchy-Schwaz inequality and Young's inequality, we further have
\begin{align*}
   -  \nabla f(\bx_k)^T D_k \varepsilon_k^b \leq \frac{1}{2}\|D_k^{1/2}\nabla f(\bx_k)\|^2 + \frac{1}{2}\|D_k^{1/2} \varepsilon_k^b\|^2, 
\end{align*}
thus
\begin{align*}
    f(\bx_{k+1}) & \le f(\bx_k) - \left(\eta-\frac{\eta^2 L }{2} \right) \| D_k^{1/2}\nabla f(\bx_k)\|^2  - (\eta - L \eta^2 M_D) \nabla f(\bx_k)^T D_k \varepsilon_k^u \\
    & \quad + \frac{1}{2} (\eta - L \eta^2 M_D) \|D_k^{1/2}\nabla f(\bx_k)\|^2 + \frac{1}{2}(\eta - L \eta^2 M_D) \|D_k^{1/2} \varepsilon_k^b\|^2 + \frac{\eta^2 L }{2} M_D^2 \|\varepsilon_k^u+\varepsilon_k^b\|^2.
\end{align*}
Now we take expectation conditional on $\bx_k$ and obtain
\begin{align*}
    \mathbb{E}_k[ f(\bx_{k+1}) ] & \le f(\bx_k) - \frac{1}{2} \eta \mathbb{E}_k[ \| D_k^{1/2}\nabla f(\bx_k)\|^2 ] - (\eta - L \eta^2 M_D) \nabla f(\bx_k)^T \mathbb{E}_k[ D_k \varepsilon_k^u] \\
    & \quad + \frac{1}{2}\eta M_D  \|\varepsilon_k^b\|^2  + \frac{\eta^2 L }{2} M_D^2 \mathbb{E}_k[ \|\varepsilon_k^u\|^2],
\end{align*}
where we have used the relation that $M_D>1$ and $\mathbb{E}[\varepsilon_k^u]=0$ and $\varepsilon_k^b$ is deterministic conditional on $\bx_k$.

\end{proof}

\begin{lemma}
Under Assumptions \ref{assumption2}, \ref{assumption3}, and \ref{assumption5}, when $\|\nabla f(\bx_k)\|<\tau/(2M_D)$, without knowing more structural information of $D_k$, we can only ensure
\begin{align*}
      \mathbb{E}_k[ f(\bx_{k+1}) ] \le f(\bx_k) - \frac{1}{4} \eta m_D \|\nabla f(\bx_k)\|^2 + \eta M_D \Upsilon_1^2 \sigma^{2p}\tau^{-2p+2} +  \eta M_D \Upsilon_2 \tau^{2-p}\sigma^p .
\end{align*}
\end{lemma}
\begin{proof}

Without explicit construction of $D_k$, we can only upper-bound this error term. We apply Hölder's inequality and obtain
\begin{align*}
     - (\eta - L \eta^2 M_D) \nabla f(\bx_k)^T \mathbb{E}_k[ D_k \varepsilon_k^u] \leq \frac{1}{4}(\eta - L \eta^2 M_D) \mathbb{E}_k[\|D_k^{1/2}\nabla f(\bx_k)\|^2]  + (\eta - L \eta^2 M_D) M_D \mathbb{E}_k[ \| \varepsilon_k^u\|^2].
\end{align*}
Therefore, we have
\begin{align*}
    \mathbb{E}_k[ f(\bx_{k+1}) ]  & \le f(\bx_k) - \frac{1}{4} \eta \mathbb{E}_k[ \| D_k^{1/2}\nabla f(\bx_k)\|^2 ] + \eta M_D \mathbb{E}_k[ \| \varepsilon_k^u\|^2]  + \frac{1}{2}\eta M_D  \|\varepsilon_k^b\|^2\\
    & \le f(\bx_k) - \frac{1}{4} \eta m_D \mathbb{E}_k[ \|\nabla f(\bx_k)\|^2 ] + \eta M_D \mathbb{E}_k[ \| \varepsilon_k^u\|^2]  + \frac{1}{2}\eta M_D  \|\varepsilon_k^b\|^2
\end{align*}

When $\|\nabla f(\bx_k)\|<\frac{\tau}{2M_D}$, we have $ \| \varepsilon_k^b\|^2 \leq \Upsilon_1^2\sigma^{2p}\tau^{-2p+2}$ and $\mathbb{E}_k[ \| \varepsilon_k^u\|^2]\leq \Upsilon_2\tau^{2-p}\sigma^p$. Plugging in them to the above inequality yields
\begin{align*}
    \mathbb{E}_k[ f(\bx_{k+1}) ] \le f(\bx_k) - \frac{1}{4} \eta m_D \|\nabla f(\bx_k)\|^2 + \eta M_D \Upsilon_1^2 \sigma^{2p}\tau^{-2p+2} +  \eta M_D \Upsilon_2 \tau^{2-p}\sigma^p .
\end{align*}
We thus complete the proof.
\end{proof}

Next we consider the case when $\|\nabla f(\bx_k)\|\geq \frac{1}{2M_D}\tau$.

\begin{lemma}\label{lemmac8}
    Under Assumptions \ref{assumption2}, \ref{assumption3}, and \ref{assumption5}, when $\|\nabla f(\bx_k)\|\geq \frac{1}{2M_D}\tau$, if we set $\eta<\frac{m_D}{24L M_D^2}$ and 
    \begin{equation*}
       \tau = \max\left\{2,  4\cdot 3^{1/p}\cdot \sigma\left(\frac{M_D^{1+\frac{3}{2p}}}{m_D^{\frac{p+1}{2p}}}\right), 64\sigma\frac{M_D^{\frac{5}{2}}}{m_D^{\frac{3}{2}}}\right\},
    \end{equation*}
    we have
    \begin{align*}
    \mathbb{E}_k[ f(\bx_{k+1}) ] \le f(\bx_k) - \frac{\eta m_D}{12M_D} \|\nabla f(\bx_k)\|.
\end{align*}
\end{lemma}

\begin{proof}

By the Taylor expansion, we have
    \begin{align*}
    f(\bx_{k+1})  \le f(\bx_k) + \nabla f(\bx_k)^T \Delta \bx_k + \frac{L }{2} \|\Delta \bx_k\|^2  = f(\bx_k) - \eta  \nabla f(\bx_k)^T D_k\widehat{\barg}_k + \frac{\eta^2 L }{2} \|D_k\widehat{\barg}_k\|^2.
\end{align*}
We first consider the conditional expectation of the 
term $\nabla f(\bx_k)^T D_k\widehat{\barg}_k$. We have
\begin{align*}
\mathbb{E}_k[\nabla f(x_k)^T D_k \widehat{\barg}_k ] = \mathbb{E}_k[\nabla f(x_k)^T D_k \bar{g}_k \cdot \boldsymbol{1}_{(\|D_k\bar{g}_k\| \le \tau)}]  + \tau \cdot \mathbb{E}_k\left[\nabla f(x_k)^T D_k \frac{\bar{g}_k}{\|D_k\bar{g}_k\|} \cdot \boldsymbol{1}_{(\|D_k\bar{g}_k\| > \tau)} \right]
\end{align*}

We focusing on the first term:
\begin{align*}
& \mathbb{E}_k[\nabla f(x_k)^T D_k \bar{g}_k \cdot \boldsymbol{1}_{(\|D_k\bar{g}_k\| \le \tau)}] \\
&= \mathbb{E}_k[\nabla f(x_k)^T D_k (\nabla f(x_k) + \bar{g}_k - \nabla f(x_k)) \cdot \boldsymbol{1}_{(\|D_k\bar{g}_k\| \le \tau)} ] \\
&= \mathbb{E}_k[\nabla f(x_k)^T D_k \nabla f(x_k) + \nabla f(x_k)^T D_k (\bar{g}_k - \nabla f(x_k)) \cdot \boldsymbol{1}_{(\|D_k\bar{g}_k\| \le \tau)} ]\\
& = \mathbb{E}_k[\nabla f(x_k)^T D_k \nabla f(x_k)\cdot \boldsymbol{1}_{(\|D_k\bar{g}_k\| \le \tau)}] + \mathbb{E}_k[\nabla f(x_k)^T D_k (\bar{g}_k - \nabla f(x_k)) \cdot \boldsymbol{1}_{(\|D_k\bar{g}_k\| \le \tau)}]\\
& \geq \mathbb{E}_k[\nabla f(x_k)^T D_k \nabla f(x_k)\cdot \boldsymbol{1}_{(\|D_k\bar{g}_k\| \le \tau)}] - \mathbb{E}_k[\|D_k^{1/2}\nabla f(x_k)\| \|D_k^{1/2} (\bar{g}_k - \nabla f(x_k))\| \cdot \boldsymbol{1}_{(\|D_k\bar{g}_k\| \le \tau)} ]\\
& = \mathbb{E}_k[\nabla f(x_k)^T D_k \nabla f(x_k)\cdot \boldsymbol{1}_{(\|D_k\bar{g}_k\| \le \tau) } ] \\
& \quad - \mathbb{E}_k[\|D_k^{1/2}\nabla f(x_k)\| \|D_k^{1/2} (\bar{g}_k - \nabla f(x_k))\| \cdot \boldsymbol{1}_{(\|D_k\bar{g}_k\| \le \tau,\|D_k^{1/2}(\bar{g}_k - \nabla f(x_k))\|\leq \frac{\tau\sqrt{m_D}}{4\sqrt{M_D}})}]\\
& \quad - \mathbb{E}_k[\|D_k^{1/2}\nabla f(x_k)\| \|D_k^{1/2} (\bar{g}_k - \nabla f(x_k))\| \cdot \boldsymbol{1}_{(\|D_k\bar{g}_k\| \le \tau,\|D_k^{1/2}(\bar{g}_k - \nabla f(x_k))\|> \frac{\tau\sqrt{m_D}}{4\sqrt{M_D}})}]
\end{align*}

We first consider the case when $\|D_k\bar{g}_k\| \le \tau,\|D_k^{1/2}(\bar{g}_k - \nabla f(x_k))\|\leq \frac{\tau\sqrt{m_D}}{4\sqrt{M_D}}$.

Since $\|\nabla f(\bx_k)\|>\frac{1}{2M_D}\tau$, we have
\begin{equation*}
    \|D_k^{1/2}\nabla f(\bx_k)\|^2 = \nabla f(\bx_k)^TD_k\nabla f(\bx_k)\geq m_D\|\nabla f(\bx_k)\|^2\geq \frac{\tau^2m_D}{4M_D},
\end{equation*}
which imply
\begin{equation*}
    \|D_k^{1/2}\nabla f(\bx_k)\| \geq \frac{\tau \sqrt{m_D}}{2\sqrt{M_D}}, \quad \|D_k^{1/2}(\bar{g}_k - \nabla f(x_k))\|\leq \frac{\tau\sqrt{m_D}}{4\sqrt{M_D}} \leq \frac{1}{2}\|D_k^{1/2}\nabla f(\bx_k)\|.
\end{equation*}
Therefore we have
\begin{equation*}
    \|D_k^{1/2}\nabla f(x_k)\| \|D_k^{1/2} (\bar{g}_k - \nabla f(x_k))\|\leq \frac{1}{2}\|D_k^{1/2}\nabla f(\bx_k)\|^2.
\end{equation*}
Further, we have
\begin{align*}
& \mathbb{E}_k[\|D_k^{1/2}\nabla f(x_k)\| \|D_k^{1/2} (\bar{g}_k - \nabla f(x_k))\| \cdot \boldsymbol{1}_{(\|D_k\bar{g}_k\| \le \tau,\|D_k^{1/2}(\bar{g}_k - \nabla f(x_k))\|\leq \frac{\tau\sqrt{m_D}}{4\sqrt{M_D}})}] \\
& \leq \frac{1}{2} \mathbb{E}_k[\|D_k^{1/2}\nabla f(\bx_k)\|^2 \cdot \boldsymbol{1}_{(\|D_k\bar{g}_k\| \le \tau,\|D_k^{1/2}(\bar{g}_k - \nabla f(x_k))\|\leq \frac{\tau\sqrt{m_D}}{4\sqrt{M_D}})}] \\
& \leq \frac{1}{2} \mathbb{E}_k[\|D_k^{1/2}\nabla f(\bx_k)\|^2 \cdot \boldsymbol{1}_{(\|D_k\bar{g}_k\| \le \tau)}]
\end{align*}
Plugging this into the above formula, we have
\begin{align*}
 & \mathbb{E}_k[\nabla f(x_k)^T D_k \bar{g}_k \cdot \boldsymbol{1}_{(\|D_k\bar{g}_k\| \le \tau)}] \\
& \geq \frac{1}{2}\mathbb{E}_k[\| D_k^{1/2} \nabla f(x_k)\|^2\cdot \boldsymbol{1}_{(\|D_k\bar{g}_k\| \le \tau)}] - \mathbb{E}_k[\|D_k^{1/2}\nabla f(x_k)\| \|D_k^{1/2} (\bar{g}_k - \nabla f(x_k))\| \cdot \boldsymbol{1}_{(\|D_k^{1/2}(\bar{g}_k - \nabla f(x_k))\|> \frac{\tau\sqrt{m_D}}{4\sqrt{M_D}})}]\\
& \geq \frac{1}{2}\mathbb{E}_k[\| D_k^{1/2} \nabla f(x_k)\|^2]\cdot \mathbb{E}_k[\boldsymbol{1}_{(\|D_k\bar{g}_k\| \le \tau)}] - \sqrt{M_D}\|\nabla f(x_k)\| \mathbb{E}_k[\|D_k^{1/2} (\bar{g}_k - \nabla f(x_k))\| \cdot \boldsymbol{1}_{(\|D_k^{1/2}(\bar{g}_k - \nabla f(x_k))\|> \frac{\tau\sqrt{m_D}}{4\sqrt{M_D}})}]\\
& = \frac{p_g}{2}\mathbb{E}_k[\| D_k^{1/2} \nabla f(x_k)\|^2] - \sqrt{M_D}\|\nabla f(x_k)\| \mathbb{E}_k[\|D_k^{1/2} (\bar{g}_k - \nabla f(x_k))\| \cdot \boldsymbol{1}_{(\|D_k^{1/2}(\bar{g}_k - \nabla f(x_k))\|> \frac{\tau\sqrt{m_D}}{4\sqrt{M_D}})}].
\end{align*}

Now we bound $\mathbb{E}_k[\|D_k^{1/2} (\bar{g}_k - \nabla f(x_k))\| \cdot \boldsymbol{1}_{(\|D_k^{1/2}(\bar{g}_k - \nabla f(x_k))\|> \frac{\tau\sqrt{m_D}}{4\sqrt{M_D}})}]$.

\begin{align*}
    M_D^{p/2} \sigma^p & \geq M_D^{p/2}\mathbb{E}_k[\|\barg_k-\nabla f(\bx_k)\|^p]\geq \mathbb{E}_k[\|D_k^{1/2}(\barg_k-\nabla f(\bx_k))\|^p] \\
    & = \mathbb{E}_k[\|D_k^{1/2}(\barg_k-\nabla f(\bx_k))\| \cdot \|D_k^{1/2}(\barg_k-\nabla f(\bx_k))\|^{p-1}] \\
    & \geq \mathbb{E}_k[\|D_k^{1/2}(\barg_k-\nabla f(\bx_k))\| \cdot \|D_k^{1/2}(\barg_k-\nabla f(\bx_k))\|^{p-1}\cdot \boldsymbol{1}_{(\|D_k^{1/2}(\bar{g}_k - \nabla f(x_k))\|> \frac{\tau\sqrt{m_D}}{4\sqrt{M_D}})}] \\
    & \geq \mathbb{E}_k[\|D_k^{1/2}(\barg_k-\nabla f(\bx_k))\| \cdot \left(\frac{\tau\sqrt{m_D}}{4\sqrt{M_D}}\right)^{p-1}\cdot \boldsymbol{1}_{(\|D_k^{1/2}(\bar{g}_k - \nabla f(x_k))\|> \frac{\tau\sqrt{m_D}}{4\sqrt{M_D}})}].
\end{align*}
Rearranging the terms yields
\begin{align*}
    \mathbb{E}_k[\|D_k^{1/2}(\barg_k-\nabla f(\bx_k))\|  \cdot \boldsymbol{1}_{(\|D_k^{1/2}(\bar{g}_k - \nabla f(x_k))\|> \frac{\tau\sqrt{m_D}}{4\sqrt{M_D}})}] \leq \frac{M_D^{p/2} \sigma^p}{\left(\frac{\tau\sqrt{m_D}}{4\sqrt{M_D}}\right)^{p-1}}.
\end{align*}
Thus, we have
\begin{align*}
    \mathbb{E}_k[\nabla f(x_k)^T D_k \bar{g}_k \cdot \boldsymbol{1}_{(\|D_k\bar{g}_k\| \le \tau)}]
    \geq \frac{p_g}{2}\mathbb{E}_k[\| D_k^{1/2} \nabla f(x_k)\|^2] -  \|\nabla f(\bx_k)\|\frac{M_D^{p/2+1} \sigma^p}{\left(\frac{\tau\sqrt{m_D}}{4\sqrt{M_D}}\right)^{p-1}}.
\end{align*}

Next, we examine $\tau \cdot \mathbb{E}_k\left[\nabla f(x_k)^T D_k \frac{\bar{g}_k}{\|D_k\bar{g}_k\|} \cdot \boldsymbol{1}_{(\|D_k\bar{g}_k\| > \tau)} \right]$.

We first consider the case $\|D_k^{1/2}\nabla f(\bx_k)\|\geq 2\|D_k^{1/2}(\barg_k-\nabla f(\bx_k))\|$, where we have
\begin{align*}
    \nabla f(\bx_k)^TD_k\barg_k 
    &= \nabla f(\bx_k)^TD_k\nabla f(\bx_k) +\nabla f(\bx_k)^TD_k (\barg_k - \nabla f(\bx_k)) \\
    & \geq \nabla f(\bx_k)^TD_k\nabla f(\bx_k) - \|D_k^{1/2}\nabla f(\bx_k)\|\|D_k^{1/2} (\barg_k - \nabla f(\bx_k))\| \\
    & \geq \nabla f(\bx_k)^TD_k\nabla f(\bx_k) - \frac{1}{2}\|D_k^{1/2}\nabla f(\bx_k)\|^2 \\
    & = \frac{1}{2}\|D_k^{1/2}\nabla f(\bx_k)\|^2.
\end{align*}

In addition, we have
\begin{align*}
    \|D_k^{1/2}(\nabla f(\bx_k) + \barg_k - \nabla f(\bx_k))\| 
    & \leq \|D_k^{1/2}\nabla f(\bx_k)\| + \| D_k^{1/2}(\barg_k - \nabla f(\bx_k))\| \\
    & \leq \frac{3}{2}\|D_k^{1/2}\nabla f(\bx_k)\|.
\end{align*}

Combining them together, we have
\begin{align*}
    \nabla f(\bx_k)^TD_k\frac{\barg_k}{\|D_k\barg_k\|} & = \frac{\nabla f(\bx_k)^TD_k(\nabla f(\bx_k)+\barg_k-\nabla f(\bx_k))}{\|D_k(\nabla f(\bx_k)+\barg_k-\nabla f(\bx_k))\|}\\
    & \geq \frac{\nabla f(\bx_k)^TD_k(\nabla f(\bx_k)+\barg_k-\nabla f(\bx_k))}{\sqrt{M_D}\|D_k^{1/2}(\nabla f(\bx_k)+\barg_k-\nabla f(\bx_k))\|}\\
    & \geq \frac{1}{3\sqrt{M_D}}\|D_k^{1/2}\nabla f(\bx_k)\|.
\end{align*}

When $\|D_k^{1/2}\nabla f(\bx_k)\|<2\|D_k^{1/2}(\barg_k-\nabla f(\bx_k))\|$, we have
\begin{align*}
    \nabla f(\bx_k)^TD_k\frac{\barg_k}{\|D_k\barg_k\|} & \geq - \| D_k^{1/2}\nabla f(\bx_k)\| \frac{\| D_k^{1/2}\barg_k\|}{\|D_k\barg_k\|} \geq - \| D_k^{1/2}\nabla f(\bx_k)\| \frac{\| D_k\barg_k\|}{\sqrt{m_D} \|D_k\barg_k\|}\\
   & \geq - \frac{1}{\sqrt{m_D}}\| D_k^{1/2}\nabla f(\bx_k)\|\\
   & =  \frac{1}{3\sqrt{M_D}}\| D_k^{1/2}\nabla f(\bx_k)\| - \left( \frac{1}{3\sqrt{M_D}} + \frac{1}{\sqrt{m_D}}\right)\| D_k^{1/2}\nabla f(\bx_k)\| \\
   & \geq \frac{1}{3\sqrt{M_D}}\| D_k^{1/2}\nabla f(\bx_k)\| - 2\left( \frac{1}{3\sqrt{M_D}} + \frac{1}{\sqrt{m_D}}\right)\| D_k^{1/2}(\barg_k-\nabla f(\bx_k))\|.
\end{align*}
Combining the above formulas, we have
\begin{align*}
   & \tau \cdot \mathbb{E}_k\left[\nabla f(x_k)^T D_k \frac{\bar{g}_k}{\|D_k\bar{g}_k\|} \cdot \boldsymbol{1}_{(\|D_k\bar{g}_k\| > \tau)} \right] \\
   & \geq  
   \tau \cdot \mathbb{E}_k\left[\left(\frac{1}{3\sqrt{M_D}}\| D_k^{1/2}\nabla f(\bx_k)\| - 2\left( \frac{1}{3\sqrt{M_D}} + \frac{1}{\sqrt{m_D}}\right)\| D_k^{1/2}(\barg_k-\nabla f(\bx_k))\|\right) \cdot \boldsymbol{1}_{(\|D_k\bar{g}_k\| > \tau)} \right] \\
   & = \tau \frac{1}{3\sqrt{M_D}} \cdot \mathbb{E}_k\left[\| D_k^{1/2}\nabla f(\bx_k)\| \cdot \boldsymbol{1}_{(\|D_k\bar{g}_k\| > \tau) } \right]  \\
   & \quad - 2\left( \frac{1}{3\sqrt{M_D}} + \frac{1}{\sqrt{m_D}}\right) \tau \cdot \mathbb{E}_k\left[ \| D_k^{1/2}(\barg_k-\nabla f(\bx_k))\|\cdot \boldsymbol{1}_{(\|D_k\bar{g}_k\| > \tau)} \right] \\
   & \geq \frac{\tau}{3\sqrt{M_D}} \cdot \mathbb{E}_k\left[\| D_k^{1/2}\nabla f(\bx_k)\|\right] \cdot \mathbb{E}_k\left[\boldsymbol{1}_{(\|D_k\bar{g}_k\| > \tau)} \right]  \\
   & \quad - 2\left( \frac{1}{3\sqrt{M_D}} + \frac{1}{\sqrt{m_D}}\right) \tau \cdot \mathbb{E}_k\left[ \| D_k^{1/2}(\barg_k-\nabla f(\bx_k))\|\cdot \boldsymbol{1}_{(\|D_k\bar{g}_k\| > \tau) } \right] \\
   & \geq  \frac{\tau (1-p_g)}{3\sqrt{M_D}} \cdot \mathbb{E}_k\left[\| D_k^{1/2}\nabla f(\bx_k)\|\right]  - 2\tau \left( \frac{1}{3\sqrt{M_D}} + \frac{1}{\sqrt{m_D}}\right)  \cdot \mathbb{E}_k\left[ \| D_k^{1/2}(\barg_k-\nabla f(\bx_k))\|  \right] \\
   & \geq \frac{\tau (1-p_g)}{3\sqrt{M_D}} \cdot \mathbb{E}_k\left[\| D_k^{1/2}\nabla f(\bx_k)\|\right]  - 2\tau \left( \frac{1}{3\sqrt{M_D}} + \frac{1}{\sqrt{m_D}}\right) \sqrt{M_D} \cdot \sigma .
\end{align*}

Combining the above derivations, we have
\begin{align*}
    \mathbb{E}_k[\nabla f(x_k)^T D_k \widehat{\barg}_k ] & = \mathbb{E}_k[\nabla f(x_k)^T D_k \bar{g}_k \cdot \boldsymbol{1}_{(\|D_k\bar{g}_k\| \le \tau)}]  + \tau \cdot \mathbb{E}_k\left[\nabla f(x_k)^T D_k \frac{\bar{g}_k}{\|D_k\bar{g}_k\|} \cdot \boldsymbol{1}_{(\|D_k\bar{g}_k\| > \tau) } \right] \\
    & \geq \frac{p_g}{2}\mathbb{E}_k[\| D_k^{1/2} \nabla f(x_k)\|^2] -  \|\nabla f(\bx_k)\|\frac{M_D^{p/2+1} \sigma^p}{\left(\frac{\tau\sqrt{m_D}}{4\sqrt{M_D}}\right)^{p-1}}\\
    & \quad +  \frac{\tau (1-p_g)}{3\sqrt{M_D}} \cdot \mathbb{E}_k\left[\| D_k^{1/2}\nabla f(\bx_k)\|\right]  - 2\tau \left( \frac{1}{3\sqrt{M_D}} + \frac{1}{\sqrt{m_D}}\right) \sqrt{M_D} \cdot \sigma .
\end{align*}

Note that since $\|D_k^{1/2}\nabla f(\bx_k)\|\geq \sqrt{m_D}\|\nabla f(\bx_k)\|$, we have
\begin{align*}
    \mathbb{E}_k[\nabla f(x_k)^T D_k \widehat{\barg}_k ]
    & \geq \frac{p_g}{2}\sqrt{m_D}\mathbb{E}_k[\| D_k^{1/2} \nabla f(x_k)\|\|\nabla f(\bx_k)\|] -  \|\nabla f(\bx_k)\|\frac{M_D^{p/2+1} \sigma^p}{\left(\frac{\tau\sqrt{m_D}}{4\sqrt{M_D}}\right)^{p-1}} \\
    & \quad +  \frac{\tau (1-p_g)}{3\sqrt{M_D}} \cdot \mathbb{E}_k\left[\| D_k^{1/2}\nabla f(\bx_k)\|\right]  - 2\tau \left( \frac{1}{3\sqrt{M_D}} + \frac{1}{\sqrt{m_D}}\right) \sqrt{M_D} \cdot \sigma .
\end{align*}
Since $\|\nabla f(\bx_k)\|\geq \frac{\tau}{2M_D}$, we further have
\begin{align*}
    \mathbb{E}_k[\nabla f(x_k)^T D_k \widehat{\barg}_k ]
    & \geq \frac{p_g}{4M_D}\sqrt{m_D}\tau\cdot\mathbb{E}_k[\| D_k^{1/2} \nabla f(x_k)\|] -  \|\nabla f(\bx_k)\|\frac{M_D^{p/2+1} \sigma^p}{\left(\frac{\tau\sqrt{m_D}}{4\sqrt{M_D}}\right)^{p-1}} \\
    & \quad +  \frac{\tau (1-p_g)}{3\sqrt{M_D}} \cdot \mathbb{E}_k\left[\| D_k^{1/2}\nabla f(\bx_k)\|\right]  - 2\tau \left( \frac{1}{3\sqrt{M_D}} + \frac{1}{\sqrt{m_D}}\right) \sqrt{M_D} \cdot \sigma  \\
    & \geq \frac{p_g}{4M_D}\sqrt{m_D}\tau\cdot\mathbb{E}_k[\| D_k^{1/2} \nabla f(x_k)\|] -  \|\nabla f(\bx_k)\|\frac{M_D^{p/2+1} \sigma^p}{\left(\frac{\tau\sqrt{m_D}}{4\sqrt{M_D}}\right)^{p-1}} \\
    & \quad +  \frac{\tau (1-p_g)\sqrt{m_D}}{4M_D} \cdot \mathbb{E}_k\left[\| D_k^{1/2}\nabla f(\bx_k)\|\right]  - 2\tau \left( \frac{1}{3\sqrt{M_D}} + \frac{1}{\sqrt{m_D}}\right) \sqrt{M_D} \cdot \sigma  \\
    & = \tau  \frac{\sqrt{m_D}}{4M_D} \cdot \mathbb{E}_k\left[\| D_k^{1/2}\nabla f(\bx_k)\|\right] -  \|\nabla f(\bx_k)\|\frac{M_D^{p/2+1} \sigma^p}{\left(\frac{\tau\sqrt{m_D}}{4\sqrt{M_D}}\right)^{p-1}} - 2\tau \left( \frac{1}{3\sqrt{M_D}} + \frac{1}{\sqrt{m_D}}\right) \sqrt{M_D} \cdot \sigma \\
    & \geq \tau  \frac{m_D}{4M_D} \cdot \|\nabla f(\bx_k)\| -  \|\nabla f(\bx_k)\|\frac{M_D^{p/2+1} \sigma^p}{\left(\frac{\tau\sqrt{m_D}}{4\sqrt{M_D}}\right)^{p-1}} - 2\tau \left( \frac{1}{3\sqrt{M_D}} + \frac{1}{\sqrt{m_D}}\right) \sqrt{M_D} \cdot \sigma.
\end{align*}
Since we set 
\begin{equation*}
    \tau = \max\left\{2,  4\cdot 3^{1/p}\cdot \sigma\left(\frac{M_D^{1+\frac{3}{2p}}}{m_D^{\frac{p+1}{2p}}}\right), 64\sigma\frac{M_D^{\frac{5}{2}}}{m_D^{\frac{3}{2}}}\right\},
\end{equation*}
we  have
\begin{equation*}
    \frac{M_D^{p/2+1} \sigma^p}{\left(\frac{\tau\sqrt{m_D}}{4\sqrt{M_D}}\right)^{p-1}} \leq \frac{\tau m_D}{12M_D}, \quad 2\tau \left( \frac{1}{3\sqrt{M_D}} + \frac{1}{\sqrt{m_D}}\right) \sqrt{M_D} \cdot \sigma \leq \frac{\tau^2m_D}{24M_D^2}.
\end{equation*}
Therefore, we have
\begin{align*}
    \mathbb{E}_k[\nabla f(x_k)^T D_k \widehat{\barg}_k ]
    & \geq   \frac{\tau m_D}{4M_D} \cdot \| \nabla f(\bx_k)\| -  \frac{\tau m_D}{12M_D}\|\nabla f(\bx_k)\| - \frac{\tau^2m_D}{24M_D^2} \\
    & \geq \frac{\tau m_D}{4M_D} \cdot \| \nabla f(\bx_k)\| -  \frac{\tau m_D}{6M_D}\|\nabla f(\bx_k)\|\\
    & = \frac{\tau m_D}{12M_D}\|\nabla f(\bx_k)\|.
\end{align*}
Plugging the above result to the Taylor expansion, we obtain
 \begin{align*}
\mathbb{E}_k[f(\bx_{k+1})]  & \leq f(\bx_k) - \eta  \mathbb{E}_k[\nabla f(\bx_k)^T D_k\widehat{\barg}_k] + \frac{\eta^2 L }{2} \mathbb{E}_k[\|D_k\widehat{\barg}_k\|^2]\\
& \leq - \eta \frac{\tau m_D}{12M_D}\|\nabla f(\bx_k)\| + \frac{\eta^2 L }{2}\tau^2 \\
& \leq - \eta \frac{\tau m_D}{12M_D}\|\nabla f(\bx_k)\| + \eta^2 L  \tau M_D \|\nabla f(\bx_k)\| \\
& \leq - \eta \frac{\tau m_D}{24M_D}\|\nabla f(\bx_k)\| \\
& \leq - \eta \frac{ m_D}{12M_D}\|\nabla f(\bx_k)\| .
\end{align*}
We thus complete the proof.
\end{proof}

\subsection{Proof of Example \ref{example} }\label{proof_example}

    Since $\nabla f(\bx_k) = C$ for some very small constant $C>0$, and the gradient estimate only takes two discrete values, then when $\barg_k>0$, we can infer that the noise is $+\sigma$, while when $\barg_k<0$, we can infer that the noise is $-\sigma$. Moreover, it is easy to verify that Assumption \ref{assumption3} is satisfied.

    Since $\tau>>\sigma$, thus the gradient estimate is not clipped, and thus
    \begin{equation*}
        \varepsilon_k^u= \barg_k - \mathbb{E}_k[\barg_k]=\pm \sigma, \quad \mathbb{E}_k[\|\varepsilon_k^u\|^2]=\sigma^2.
    \end{equation*}
Let's compute the expectation $\mathbb{E}_k[ D_k \epsilon_k^u ]$:$$\begin{aligned}
\mathbb{E}_k[ D_k \epsilon_k^u ] &= 0.5 \cdot (D_k|_{\epsilon=\sigma} \cdot \sigma) + 0.5 \cdot (D_k|_{\epsilon=-\sigma} \cdot (-\sigma)) \\
&= 0.5 \cdot (m_D \cdot \sigma) + 0.5 \cdot (M_D \cdot (-\sigma)) \\
&= 0.5 \sigma (m_D - M_D)\\
& = \mathcal{O}(\sigma^2).
\end{aligned}$$
Recall that $\nabla f(\bx_k)$ is a constant, we have
\begin{equation*}
    -(\eta - L\eta^2 M_D)\nabla f(\bx_k)^T\mathbb{E}_k[D_k\epsilon_k^u] = \mathcal{O}(\eta\sigma^2)=\mathcal{O}(\eta\mathbb{E}_k[\|\epsilon_k^u\|^2]).
\end{equation*}
We complete the proof.

\section{Vector-valued Burkholder-type inequality}\label{append_E}

\begin{lemma}[Vector-valued Burkholder-type inequality]\label{append_C:lemma1}
Let $X_k, k=0,1,\dots,T-1$, be random variables with $\mathbb{E}[X_k\mid X_{0:k-1}]=\boldsymbol{0}$ and $\mathbb{E}[\|X_k\|^{p}]<\infty$ for any $k\geq 0$ and some $p \in (1,2]$, then 
\begin{equation*}
\mathbb{E}\left[ \left\| \sum_{k=0}^{T-1} X_k \right\| ^{p} \right] \leq 2^{2-p} \sum_{k=0}^{T-1} \mathbb{E}[\|X_k\|^{p}].
\end{equation*}
\end{lemma}	

 The proof is a combination of the proof of the scalar-valued Burkholder-type inequality in \cite{Fang2025High} Lemma B.1 and the proof in \cite{Hubler2024Gradient}.

We first revisit the Lemma B.1 in \cite{Fang2025High}. We include the proof in Appendix \ref{proofB1} for completeness.

\begin{lemma}[Scalar-valued Burkholder-type inequality]\label{append_C:lemma2}
Let $X_k, k=0,1,\dots,T-1$, be random variables with $\mathbb{E}[X_k\mid X_{0:k-1}]=0$ and $\mathbb{E}[|X_k|^{p}]<\infty$ for any $k\geq 0$ and some $p \in (1,2]$, then 
\begin{equation}\label{append:eq2}
\mathbb{E}\left[ \left| \sum_{k=0}^{T-1} X_k \right| ^{p} \right] \leq 2^{2-p} \sum_{k=0}^{T-1} \mathbb{E}[|X_k|^{p}].
\end{equation}
\end{lemma}

\subsection{Proof of Lemma \ref{append_C:lemma1}}
\begin{proof}
    Following (Kornilov et al., 2024), we define $g \sim \mathcal{N}(\boldsymbol{0}, I)$ and $y_j := g^T X_j$, where $g$ is independent of $X_j$. We need to verify that $y_j$ defined this way satisfies the conditions in Lemma \ref{append_C:lemma2}.
    
    We define the sigma algebra $\mathcal{H}_1 \coloneqq \sigma(y_{j-1}, \dots, y_1)$ and $\mathcal{H}_2 := \sigma(X_{j-1}, \dots, X_1, g)$. Since $\mathcal{H}_1 \subset \mathcal{H}_2$, by the tower property, we have
    \begin{equation*}
        \mathbb{E} [g^T X_j | \mathcal{H}_1] = \mathbb{E} [\mathbb{E} [g^T X_j | \mathcal{H}_2] | \mathcal{H}_1] = \mathbb{E} [g^T \mathbb{E} [X_j | \mathcal{H}_2] | \mathcal{H}_1] = 0,
    \end{equation*}
where the last equality holds by independence of $X_j$ and $g$, and the assumption that $\mathbb{E}[X_j \mid X_{1:j-1}] = 0$. 

Next, we verify that $\mathbb{E}[|y_j|^p] < \infty$. Note that $g^T a \sim \mathcal{N}(0, \|a\|^2)$ for any vector $a \in \mathbb{R}^d$, we have $\mathbb{E}[|y_j|^p | X_j] = \mathbb{E}[|g^T X_j|^p | X_j] = C(p) \|X_j\|^p$, with $C(p) = 2^{p/2} \Gamma(\frac{p+1}{2})/\sqrt{\pi}$, where we used the $p$-th absolute moment of normal distribution applied to a random variable $g^T X_j$ given $X_j$. Taking full expectation, we get
\begin{equation}\label{append:eq3}
\mathbb{E} [|y_j|^p] = C(p) \mathbb{E} [\|X_j\|^p] < \infty.
\end{equation}

We have verified that the sequence $y_j$ satisfies the conditions in Lemma \ref{append_C:lemma2}. 

Using the $p$-th moment of normal distribution applied to $g^T \sum_{k=0}^{T-1}X_k$ given $\sum_{k=0}^{T-1}X_k$, we have
\begin{equation*}
    C(p) \mathbb{E} \left[\left\|\sum_{k=0}^{T-1}X_k\right\|^p\right] = \mathbb{E} \left[\mathbb{E} \left[\left|g^T \sum_{k=0}^{T-1}X_k\right|^p \mid \sum_{k=0}^{T-1}X_k\right]\right] = \mathbb{E} \left[ \left| \sum_{j=1}^{n} y_j \right|^p \right] \le 2^{2-p} \sum_{j=1}^{n} \mathbb{E} [|y_j|^p],
\end{equation*}
where in the last step we used \eqref{append:eq2}. We then apply \eqref{append:eq3} and cancel $C(p)$ and complete the proof.
\end{proof}

\subsection{Proof of Lemma \ref{append_C:lemma2}}\label{proofB1}
\begin{proof}

To prove the result, we fist show the following inequality: \begin{equation}\label{append:eq4}
|a + b|^p \leq |a|^p + p \cdot \text{sgn}(a)\cdot |a|^{p-1} b + 2^{2-p} |b|^p.
\end{equation}
When $a=0$, the inequality holds trivially as $2^{2-p}\geq 1$ for $p\in (1,2]$. Moreover, when $p=2$, the inequality also holds trivially by the observation that $\text{sgn}(a)|a|=a$. In what follows, we consider $a\neq 0$ and $p\in (1,2)$. We divide $|a|^p$ on both sides and get
\begin{equation}\label{append:eq5}
\left|1+ \frac{b}{a}\right|^p \leq 1 + p \frac{b}{a} + 2^{2-p}\left|\frac{b}{a}\right|^p,
\end{equation}
where we use the observation that $\text{sgn}(a)/|a| = 1/a$. Denoting $x=b/a$, it suffices to show that $f(x)= 1 + p x + 2^{2-p}\left|x\right|^p-\left|1+ x \right|^p \geq 0$. To this end, we consider three cases. \textbf{Case A:} $b/a>0$, \textbf{Case B:} $-1<b/a\leq 0$, and \textbf{Case C:} $b/a \leq -1$.

\vskip4pt
\noindent$\bullet$ \textbf{Case A:} $b/a>0$: We have $f(x) = 1+px+2^{2-p}x^p-(1+x)^p$ with $f(0) = 0, f'(0)=0$~and \begin{equation*}
f''(x) = p(p-1) \left[ \left( \frac{2}{x} \right)^{2-p} - \left( \frac{1}{1+x} \right)^{2-p} \right].
\end{equation*}
Since $1 + x > x/2 $ for all $x > 0$, when $p < 2$, we have \begin{equation*}
\left( \frac{2}{x} \right)^{2-p} > \left( \frac{1}{1+x} \right)^{2-p} \implies f''(x) > 0.
\end{equation*}
Thus, for all $x > 0$, we have $f'(x) > 0$, which then implies $f(x) >  0$, and thus \eqref{append:eq4} holds.

\vskip4pt
\noindent $\bullet$ \textbf{Case B:} $-1<b/a\leq 0$: We substitute $y=1+x$ so that $ 0 \leq y \leq 1 $. Then,
\begin{equation*}
f(x)=f(y-1)= 2^{2-p} (1 - y)^p + p (y - 1) + 1 - y^p \geq (1 - y)^p + p (y - 1) + 1 - y^p \coloneqq g(y).
\end{equation*}  
Then, we have $g(1)=0$, $g'(0)=g'(1)=0$, and 
\begin{equation*}
g''(y)= p (p-1) \left[ \left( \frac{1}{1 - y} \right)^{2-p} - \left( \frac{1}{y} \right)^{2-p} \right].
\end{equation*}
Note that $g''(y)$ has exactly one root at $y = 1/2$. Moreover, $ g''(y) < 0 $ for $y < 1/2 $ and $ g''(y) > 0 $ for $y >1/2$, so $ g'(y) $ has a minimum at $ y = \frac{1}{2} $. Since $ g'(0) = g'(1) = 0 $, we conclude that $ g'(y) < 0 $ for all $ 0 < y < 1 $. As $g(y)$ is monotonically decreasing on $(0,1)$ and $g(1)=0$, we have $ g(y) \geq 0 $ for all $0 \leq y \leq 1 $. Therefore, $f(x)\geq 0$.

\vskip4pt
\noindent$\bullet$ \textbf{Case C:} $b/a \leq -1$: We substitute $z = -1 - x $ so that $z\geq 0$. Then $f(x)=f(-z-1)\eqqcolon h(z)$ and \vspace{-0.2cm}
\begin{equation*}
h(z) = 2^{2-p} (1 + z)^p - p (1 + z) + 1 - z^p.
\end{equation*}
Then, we have $h(0) \geq 0, h'(1) = 0$, and  \vspace{-0.2cm}
\begin{equation*}
h'(z) = p (p-1) \left[ \left( \frac{2}{1 + z} \right)^{2-p} - \left( \frac{1}{z} \right)^{2-p} \right].
\end{equation*}
Note that $ h''(z) = 0 $ only at $ z = 1$. Moreover, $ h''(z) < 0 $ for $ z < 1 $ and $ h''(z) > 0 $ for $ z > 1 $, so $ h'(z) $ has a minimum at $z = 1 $. Since $ h'(1) = 0 $, we know $ h'(z) \geq 0 $ for all $z \geq 0 $, and hence $ h(z) \geq 0 $ for all $z \geq 0$. Combining the above three cases, we prove \eqref{append:eq4}.

Next, we prove \eqref{append:eq2} by induction. The inequality holds trivially for $T=1$, as $2^{2-p} \geq 1$ for $p \in (1,2] $. Suppose the inequality holds for $T=n$, we now consider $T=n+1$.~Applying \eqref{append:eq4} with $a =  \sum_{k=0}^{n-1} X_k $ and $b = X_{n}$, we have
\begin{equation*}
\left| \sum_{k=0}^{n} X_k \right| ^{p} \leq  \left| \sum_{k=0}^{n-1} X_k \right| ^{p} + p\cdot \text{sgn}(\sum_{k=0}^{n-1} X_k) \cdot \left| \sum_{k=0}^{n-1} X_k \right|^{p-1} \cdot X_{n} + 2^{2-p}  |X_{n}|^p .
\end{equation*}
Taking expectation on both sides, we have
\begin{equation*}
\mathbb{E}\left[ \left| \sum_{k=0}^{n} X_k \right| ^{p} \right] \leq  \mathbb{E}\left[  \left| \sum_{k=0}^{n-1} X_k \right| ^{p} \right] + p \mathbb{E}\left[  \text{sgn}(\sum_{k=0}^{n-1} X_k) \left| \sum_{k=0}^{n-1} X_k \right|^{p-1} \cdot X_{n}  \right]  + 2^{2-p}  \mathbb{E}\left[  |X_{n}|^p \right]  .
\end{equation*}
Using $\mathbb{E}[X_{n}\mid X_{0:n-1}]=0$, we get
\begin{equation*}
\mathbb{E}\left[ \left| \sum_{k=0}^{n} X_k \right| ^{p} \right] \leq  \mathbb{E}\left[  \left| \sum_{k=0}^{n-1} X_k \right| ^{p} \right] + 2^{2-p}  \mathbb{E}\left[  |X_{n}|^p \right] .
\end{equation*}
We complete the proof by induction.
\end{proof}

Now we extend the above one to vector valued.


\end{document}